\documentclass[review,hidelinks,onefignum,onetabnum]{siamart220329}
\usepackage{caption}
\usepackage{subcaption}
\usepackage{multirow}
\usepackage[ruled]{algorithm2e}
\DeclareMathOperator*{\argmin}{arg\,min}

\usepackage{lipsum}
\usepackage{amsfonts}
\usepackage{graphicx}
\usepackage{epstopdf}
\ifpdf
  \DeclareGraphicsExtensions{.eps,.pdf,.png,.jpg}
\else
  \DeclareGraphicsExtensions{.eps}
\fi

\newsiamremark{remark}{Remark}
\newsiamremark{hypothesis}{Hypothesis}
\crefname{hypothesis}{Hypothesis}{Hypotheses}
\newsiamthm{claim}{Claim}

\headers{The SCS ALGORITHM FOR
KERNEL SVM}{DI ZHANG, AND SUVRAJEET SEN}

\title{The Stochastic Conjugate Subgradient Algorithm for Kernel Support Vector Machines\thanks{Submitted to the editors DATE.
\funding{This paper is supported, in part, by a grant AFOSR FA9550-20-1-0006.}}}

\author{Di Zhang\thanks{Department of Industrial and System Engineering, University of Southern California, Los Angeles, CA 
(\email{dzhang22@usc.edu}).}
\and Suvrajeet Sen\thanks{Department of Industrial and System Engineering, University of Southern California, Los Angeles, CA (\email{suvrajes@usc.edu}).}}

\usepackage{amsopn}

\ifpdf
\hypersetup{
  pdftitle={SCS Algorithm for Kernel SVM},
  pdfauthor={Suvrajeet Sen, Di Zhang}
}
\fi

\usepackage{array}
\newcolumntype{L}[1]{>{\raggedright\let\newline\\\arraybackslash\hspace{0pt}}m{#1}}
\newcolumntype{C}[1]{>{\centering\let\newline\\\arraybackslash\hspace{0pt}}m{#1}}
\newcolumntype{R}[1]{>{\raggedleft\let\newline\\\arraybackslash\hspace{0pt}}m{#1}}
\begin{document}

\maketitle

\begin{abstract}
Stochastic First-Order (SFO) methods have been a cornerstone in addressing a broad spectrum of modern machine learning (ML) challenges. However, their efficacy is increasingly questioned, especially in large-scale applications where empirical evidence indicates potential performance limitations. In response, this paper proposes an innovative method specifically designed for kernel support vector machines (SVMs). This method not only achieves faster convergence per iteration but also exhibits enhanced scalability when compared to conventional SFO techniques. Diverging from traditional sample average approximation strategies that typically frame kernel SVM as an “all-in-one” Quadratic Program (QP), our approach adopts adaptive sampling. This strategy incrementally refines approximation accuracy on an ``as-needed'' basis. Crucially, this approach also inspires a decomposition-based algorithm, effectively decomposing parameter selection from error estimation, with the latter being independently determined for each data point. To exploit the quadratic nature of the kernel matrix, we introduce a stochastic conjugate subgradient method. This method preserves many benefits of first-order approaches while adeptly handling both nonlinearity and non-smooth aspects of the SVM problem. Thus, it extends beyond the capabilities of standard SFO algorithms for non-smooth convex optimization. The convergence rate of this novel method is thoroughly analyzed within this paper. Our experimental results demonstrate that the proposed algorithm not only maintains but potentially exceeds the scalability of SFO methods. Moreover, it significantly enhances both speed and accuracy of the optimization process.
\end{abstract}



\section{Introduction}

Kernel support vector machines (SVM), a subset of supervised learning models, are instrumental in data-driven classification tasks~\cite{scholkopf2002}. Consider a dataset of covariates represented as $S=\left\{ (z_i,w_i) \right\}_{i=1}^{m}$, where the pairs correspond to features and their associated labels respectively. The primary objective is to predict the label $W$ given a random feature $Z$ utilizing the classifier $\beta$. In scenarios where the feature vector $Z$ is exceedingly large or the data deviates from linear separability, it becomes necessary to apply a non-linear transformation to the feature vector $Z$, commonly referred to as kernel mapping $\phi$. Consequently, the kernel SVM problem is formulated as:

\begin{equation}
\label{Loss-function}
\min_{\beta} g(\beta) := \frac{1}{2}||\beta||^2+E [\max\{0,1-W \langle \beta, \phi(Z) \rangle\}],
\end{equation}

\noindent where the expectation $E$ is taken with respect to the joint distribution of $(Z,W)$. Directly addressing the problem as defined in Equation \eqref{Loss-function} is infeasible due to two significant challenges: (a) the joint distribution of the pair $(Z,W)$ is typically unknown, and (b) the kernel mapping $\phi$ is not directly accessible. To overcome the first challenge (a), the focus generally shifts towards resolving an empirical risk minimization (ERM) problem, specifically through a sample average approximation (SAA) approach. With regard to the second challenge (b), while direct access to $\phi$ may not be available, the inner product within the kernel mapping, $\langle \phi(z_i), \phi(z_j) \rangle$, may be available.  This provides a valuable workaround~\cite{scholkopf2002}. Essentially, one uses the Representer Theorem\cite{KG1971} to express $\beta$ in the form of $\beta=\sum_{i=1}^{m} \alpha_i \phi(z_i)$. This transformation enables a finite sample kernel SVM approximation as follows:

\begin{equation} 
\label{Kernel-SVM}
\min_{\alpha}f_S(\alpha): = \frac{1}{2}\langle \alpha, Q \alpha \rangle + \frac{1}{m} \sum_{i=1}^{m} \max\{0,1-w_i \langle \alpha, Q_i \rangle\},
\end{equation}

\noindent where $\alpha = [\alpha_1, \alpha_2, ..., \alpha_m]$ is the decision varaible, $Q \in \mathbb{R}^{m \times m}$ is the kernel matrix with $ Q_{ij} = \langle \phi(z_i), \phi(z_j) \rangle$ and $Q_i$ refers to the $i$-th row of $Q$. It is significant to observe that so long as we have access to $Q_{ij}=\langle \phi(z_i), \phi(z_j) \rangle$, we can solve \eqref{Kernel-SVM} without having the direct knowledge of the kenel mapping $\phi$.  Notably, a variety of kernel mappings, including but not limited to polynomial, Gaussian, and Laplacian kernels, often result in a positive definite matrix $Q$~\cite{scholkopf2002}. This characteristic implies that equation \eqref{Kernel-SVM} constitutes a convex optimization problem.

Equation \eqref{Kernel-SVM} can be interpreted as an “all-in-one” Quadratic Program (aioQP). However, the computational feasibility of solving the aioQP becomes challenging with extremely large datasets, such as those containing millions of data points. To address this problem, our paper introduces an adaptive sampling method aimed at improving the quality of the approximation locally, while enhancing the accuracy of the global approximation on an “as-needed” basis. The approximation at iteration $k$, utilizing a significantly smaller sample size $|S_k| << m$, is expressed as:

\begin{equation}
\label{kenel-SVM-SAA}
\min_{\alpha}f_k(\alpha) := \frac{1}{2}\langle \alpha, Q^k \alpha \rangle + \frac{1}{|S_k|} \sum_{i=1}^{|S_k|} \max\{0,1-w_i \langle \alpha, Q_i^k \rangle\},
\end{equation}

\noindent where $\alpha = [\alpha_1, \alpha_2, …, \alpha_{|S_k|}]$ is the decision variable, $Q^k \in \mathbb{R}^{m \times m}$ is the kernel matrix with $ Q^k_{ij} = \langle \phi(z_i), \phi(z_j) \rangle$ and $Q^k_i$ refers to the $i$-th row of $Q^k$. We will demonstrate that our proposed approach not only offers enhanced accuracy but also has greater scalability compared to Stochastic First Order (SFO) methods in tackling large-scale kernel SVM problems. To lay the groundwork, we will commence by reviewing a range of classic stochastic methods, many of which are applicable to kernel SVM.




\subsection{Stochastic first-order methods}
SFO methods such as ``stochastic gradient descent" (SGD) \cite{RM1951,BB2011, NB2001} or dual coordinate ascent \cite{P1998,J1999} are very popular for solving machine learning problems. For black-box non-smooth optimization, SGD-type algorithms using secant approximations have been proposed by~\cite{YNShanbhag2020}. It is also worth noting that a recent survey~\cite{L2023} summarizes this type of optimization method comprehensively. Simple first-order gradient-based methods dominate the field for several convincing reasons: low computational cost, simplicity of implementation, and strong empirical results.   
For example, in order to solve \eqref{Kernel-SVM}, it is common to adopt a specialized version of SGD such as Pegasos~\cite{SS2011} where the subgradient calculations can be carried out very rapidly and in parallel (for a given $\alpha \in \mathbb{R}^n$). \cite{SZ2013} shows that the Pegasos algorithm obtains an $\varepsilon_p$-sub-optimal solution in time $O(1/(\lambda \varepsilon_p))$, where $\lambda$ is the condition number of $Q$. This run-time does not depend on the number of variables $n$, and therefore is favorable when $n$ is very large.  Finally, asymptotitic analysis via the work of Polyak and Juditsky~\cite{polyak1992} has shown that SFO exhibits asymptotic normality for smooth stochastic optimization and a similar result has been recently obtained in~\cite{davis2023} for the case of non-smooth stochastic convex optimization problems. 

While the aforementioned methods have low computational requirements in any one iteration, there are some disadvantages: a) the convergence rate of such algorithms depend heavily on the condition number of the Hessian of the quadratic function, b) SFO methods are not known for their numerical accuracy especially when the functions are non-smooth and, c) the convergence rate mentioned in the previous paragraph only applies under the assumption that the component functions are
smooth. Since \eqref{Kernel-SVM} has non-smooth component functions, such a convergence rate is not guaranteed.   


\subsection{Stochastic methods for general non-linear  optimization}
For general non-linear stochastic optimization, there are two broad classes of algorithms, most of which are stochastic  generalizations of smooth deterministic optimization algorithms. For gradient based direction finding,  stochastic analogs were studied by \cite{PS2020}, whereas stochastic analogs of trust region methods (STORM) appear in a series of papers \cite{B2014,CS2018,C2018}. From a computational standpoint, direction-finding algorithms may have faster iterations, whereas trust region methods may provide stronger convergence properties \cite{BS2019}. However, the aforementioned stochastic methods rely on the existence of gradients which may not be appropriate for non-smooth stochastic problems such as kernel SVM in \eqref{Kernel-SVM}. Readers interested in the non-smooth, non-convex case may refer to \cite{LS2022}.

\subsection{Stochastic conjugate gradient method}
In contrast to first-order methods, algorithms based on Conjugate Gradients (CG) provide finite termination when the objective function is a linear-quadratic deterministic function \cite{N2006}. Furthermore, CG algorithms are also used for large-scale nonsmooth convex optimization problems \cite{W1975,Y2016}. However, despite its many strengths, (e.g., fast per-iteration
convergence, frequent explicit regularization on step-size, and better parallelization than first-order
methods), the use of conjugate gradient methods for Stochastic Optimization is uncommon. Although \cite{J2018,Y2022} proposed some new stochastic conjugate gradient (SCG) algorithms, their assumptions do not support non-smooth objective functions, and hence not directly applicable to kernel SVM problems.
The following assumptions, which are common for kernel SVM models, will be imposed throughout this paper.

\begin{itemize} \label{conditions}
    \item The optimization problems \eqref{Loss-function}-\eqref{kenel-SVM-SAA} are convex and finite valued. Thus, extended real-valued functions are not permissible in our framework. As a consequence, we assume the function $f$ and $f_k$ possess the Lipschitz constant $0 < L_f < \infty$ and their sub-differentials are non-empty for every $\alpha$. Note that when the sample size is large enough, $L_f$ will represent an upper bound of the Lipschitz constant on both $f$ and $f_k$. Efficient and accurate estimation of the Lipschitz constant holds significant practical importance. Readers interested in this topic may find valuable insights in the work of~\cite{F2019,W1996}.
    
    \item The inner product of the kernel function and the classifier is bounded, i.e., for any kernel mapping $\phi(z)$ and classifier $\beta$,  $|\langle \phi(z), \beta \rangle| < M$. Intuitively, this assumption indicates that the classification error can not be unbounded, which is reasonable for kernel SVM. 
\end{itemize}



In this paper, we adopt a direction-finding approach in which the direction is obtained via a specialized problem of finding a vector with the smallest norm over a line segment. This approach was originally designed for deterministic non-smooth convex optimization problems \cite{W1974,W1975}. 

In order to solve stochastic optimization problems when functions are non-smooth (as in  \eqref{Kernel-SVM}), it is straightforward to imagine that a stochastic subgradient method could be used. In this paper, we treat \eqref{Kernel-SVM} as a stochastic piecewise linear quadratic optimization problem and propose a new stochastic conjugate subgradient (SCS) algorithm which accommodates both the curvature of a quadratic function, as well as the decomposition of subgradients by data points, as is common in stochastic programming \cite{HS1991,HS1994}. This combination enhances the power of stochastic first-order approaches without the additional burden of second-order approximations. In other words, our method shares the spirit of Hessian Free (HF) optimization, which has attracted some attention in the machine learning community \cite{MJ2010,CScheinberg2017}. As our theoretical analysis and the computational results reveal, the new method provides a ``sweet-spot'' at the intersection of speed, accuracy, and scalability of algorithms (see Theorems \ref{L^0} and \ref{Convergence rate}).



\subsection{Contributions}

(1) Our methodology extends Wolfe's deterministic conjugate subgradient method \cite{W1975} into the stochastic domain. While a straightforward application of \cite{W1975} to stochastic programming (SP) problems using Sample Average Approximation (SAA) \cite{Shapiro2003} might be considered, our focus is on solving Kernel SVM problems with extremely large datasets such as Hepmass with 3.5 million samples. Consequently, we employ an adaptive sampling strategy over a deterministic finite sum approximation. This approach bears some resemblance to Stochastic Decomposition (SD) \cite{HS1991, HS1994}, although SD employs a "trust-region" strategy for iterative progress. More notably though, deterministic algorithms for finite-sum SAA can impose significant computational and memory demands, especially with large datasets.  This observation is corroborated by our experimental results. To address the challenges arising from non-smoothness as well as extremely large datasets, our novel algorithm integrates several disparate notions such as adaptive sampling, decomposition, conjugate subgradients, and line search techniques. We demonstrate that this amalgamation effectively approximates solutions for kernel SVM problems with extremely large datasets.

(2) The ensuing analysis of our algorithm establishes a convergence rate of $O(1/\varepsilon^2)$ for kernel SVM, where $\varepsilon$ denotes the desired accuracy. This contrasts with the convergence rate of $O(1/\sqrt[5]{k})$ reported in \cite{YNShanbhag2020}, applicable to smooth, albeit far more demanding black-box and non-strongly convex problems. 

(3) In contrast to trust-region methods, which necessitate solving a constrained convex optimization problem in each iteration, our algorithm derives a descent direction by solving a one-dimensional convex optimization problem.

(4) Our computational experiments demonstrate that, from an optimization standpoint, the solutions produced by SCS yield lower objective function values compared with SFO methods, maintaining competitive computational times for small to medium-sized datasets. More crucially, as the dataset size increases, the efficiency and accuracy of our Stochastic Conjugate Subgradient (SCS) algorithm surpasses those of algorithms such as Pegasos, a specialized first-order methods for kernel SVM. 

The structure of this paper is as follows.  In \S\ref{SCSA}, the discussion evolves from mathematical concepts to computational algorithms. We will discuss the SCS algorithm for solving large scale kernel SVM via a combination of Wolfe's non-smooth conjugate subgradient method \cite{W1975} and adaptive sampling which helps leverage a successive sampling methodology as in Stochastic Decomposition \cite{HS1991, HS1994}. This combination not only introduces some elements of recursive estimation into kernel SVM, but also bypasses the need for constrained QP algorithms, or even the need for solving a large linear system (of equations) rapidly. In \S\ref{CA}, we will present the convergence analysis of the SCS algorithm by showing that the algorithm generates a super-martingale sequence which converges in expectation. In \S\ref{IT&ER}, we will provide some technical details about its implementation and compare the computational results with its two progenitors: the Pegasos algorithm (a SFO method) and Wolfe's algorithm~\cite{W1975}, which is a deterministic non-smooth optimization method. In the context of this comparison, we will discuss some advantages and disadvantages of adopting the SCS algorithm for solving the kernel SVM problem. Finally, our conclusions are provided in \S\ref{Con}. The notations used in this paper are summarized in Appendix \ref{notations}.

\section{Stochastic Conjugate Subgradient (SCS) Algorithm}\label{SCSA}

In this section, we focus on solving (\ref{kenel-SVM-SAA}) using the SCS  algorithm in which we combine the non-smooth CG method~\cite{W1975} with adaptive sampling and decomposition as in stochastic programming \cite{BL2011}. In this setting, decomposition refers to a treatment of the piecewise linear (max) function and the quadratic part separately in \eqref{kenel-SVM-SAA} to find the direction of change for the decision variables. We will first state the SCS algorithm (Algorithm \ref{SCSQL-KSVM}) and then explain the key concepts underlying this approach.  In the following, we let $Nr(G)$ denote the projection of $``0"$ onto any finite collection of vectors $G$.  During the course of the algorithm, we will adaptively enlarge the data set which is represented within the SVM problem.  Before the start of the $k^{th}$ iteration, the set of data available to the algorithm will be denoted $S_{k-1}$.  In iteration $k$ we will include additional data points denoted by the set $s_k$, so that $|S_k| = |S_{k-1}| + |s_k|$. The SCS algorithm consists of four important components:
\begin{itemize}
\item Sequential function approximation. In many cases we have a fixed number of data points. However, our focus is on situations where the data can be queried from an oracle sequentially. Thus, we use adaptive sampling~\cite{HS1991,HS1994,P2010}, 

\begin{equation} \label{kenel-SVM-SAA-2}
    \min_{\alpha}f_k(\alpha) = \frac{1}{2}\langle \alpha, Q^k \alpha \rangle + \frac{1}{|S_k|} \sum_{i=1}^{|S_k|} \max\{0,1-w_i \langle \alpha, Q_i^k \rangle\}.
\end{equation}

It is important to note that $w_i$ and $Q_i^k$ are based on realization of random variables $(Z,W)$. Consequently, $f_k$ is a realization of the expectation functional in  \eqref{Loss-function}.
In the remainder of this paper, we will let the kernel matrix $Q^k$ agree with the size of data set ($|S_k|$) in any iteration.  

\item Direction finding. The idea here is inspired by Wolfe's non-smooth conjugate subgradient method which uses the smallest norm of the convex combination of the previous search direction ($-d_{k-1}$) and the current subgradient ($g_k$). More specifically, we first solve the following one-dimensional QP

\begin{equation} \label{p^k}
\lambda_k^* = \argmin_{\lambda \in [0,1]}\frac{1}{2}||\lambda \cdot (-d_{k-1}) + (1-\lambda) \cdot g_k||^2.
\end{equation}

\noindent Then we can set the new search direction as
\begin{equation*}
    d_k=-\big[\lambda_k^* \cdot (-d_{k-1})+(1-\lambda_k^*) \cdot g_k\big] := -Nr(G_k),
\end{equation*}

\noindent where $G_k=\{-d_{k-1},g_k\}$, and $\lambda_k^*$ denotes the optimal weight for the convex combination.  Clearly if one fixes $\lambda = 0$, then the search direction reduces to that of the subgradient method.

\item Choice of step size. This is achieved by combining concepts of deterministic Wolfe's algorithm, stochastic trust region and line search methods~\cite{N2006, W1975}. The step actually applies Wolfe's line search criteria to the function which is a sampled realization of the expectation functional. Let $g(t) \in \partial f_{k-1}(\hat{\alpha}_k + t \cdot d_k)$ and define the intervals $L$ and $R$.
\begin{equation} \label{L&R}
    \begin{aligned}
        & L=\{t>0 \ | \ f_{k-1}(\hat{\alpha}_k + t \cdot d_k)-f_{k-1}(\hat{\alpha}_k) \leq -m_1||d_k||^2 t \},\\
        & R =\{t>0 \ | \ 0 > \langle g(t),d_k \rangle \geq -m_2||d_k||^2 \}.
    \end{aligned}
\end{equation}
\noindent The above conditions are consistent with the line-search rules in Wolfe's deterministic algorithm~\cite{W1974}. The output of the step size will satisfy two metrics: (i) Identify a set $L$ which includes points which reduce the objective function approximation $f_{k-1}$, and (ii) Identify a set $R$ for which the directional derivative estimate for $f_{k-1}$ is improved. The algorithm seeks points which belong to $L \cap R$. The combination of $L$ and $R$ is called strong Wolfe condition and it has been shown in \cite{W1975} Lemma 1 that $L \cap R$ is not empty.

\item Termination criteria. The algorithm concludes its process when $||d_k||<\varepsilon$ and $\delta_k < \delta_{min}$. As we will show in Theorem \ref{optimality}, a dimunitive value of $||d_k||$ indicates a small norm of the subgradient $g \in f(\hat{\alpha}_k)$, fulfilling the optimality condition for an unconstrained convex optimization problem. Additionally, a threshold $\delta_{min}$ is established to prevent premature termination in the initial iterations. Without this threshold, the algorithm may halt prematurely if the initial samples are correctly classified by the initial classifier $\alpha_0$, resulting in $||d_k|| < \varepsilon$. However, if $\delta_{min}$ is introduced, based on Theorem \ref{L^0}, $|S_k|$ should be chosen such that 
\begin{equation*}
    |S_k| \geq -8\log(\varepsilon/2)\cdot \frac{(M+1)^2}{\kappa^2\delta_{min}^4},
\end{equation*}   

\noindent where $\kappa$ is the $\kappa$-approximation defined in Definition \ref{kappa approximation}. This effectively mitigates any early stopping concern with high probability.
\end{itemize}

\begin{algorithm}[]
\caption{Stochastic Conjugate Subgradient (SCS) Algorithm}
\scriptsize
\label{SCSQL-KSVM}
\SetAlgoLined
\nl Set $\varepsilon > 0$, $\delta_{min} < \delta_0 < \delta_{max}$, $\eta_1 \in (0,1)$, $\eta_2 > 0$, $\gamma>1$, and set $k \leftarrow 0$. 
\label{setup}

\nl Randomly generate $S_0$ from the data set $S$ and build an initial Radial Basis Function (RBF) kernel $Q^0$ based on $S_0$.

\nl Define $f_0(\alpha)= \frac{1}{2} \langle \alpha, Q^0 \alpha \rangle +\frac{1}{|S_0|}\sum_{i=1}^{|S_0|} \max\{0,1-w_i\langle \alpha, Q^0_i \rangle \}.$ \label{SAA}

\nl Set a feasible solution $\hat{\alpha}_0 = \overrightarrow{0} \in \mathbb{R}^{s_0}$ and an initial direction $d_0 \in \partial f_0(\hat{\alpha}_0)$.

\While{$||d_k||>\varepsilon \ \textit{or} \ \delta_k>\delta_{min} $}
{
    \nl $k \leftarrow k+1$\;
   
    \nl Obtain $g_{k} \in \partial {f}_{k-1}(\hat{\alpha}_{k-1})$, $G_k= \{ -d_{k-1},g_k\}$ and $d_k =-Nr(G_k)$\; \label{conjugate}
   
    \nl Apply Algorithm \ref{LSA} to find step size $t_k$ \;
    
    \nl Set $\alpha_k = \hat{\alpha}_{k-1} + t_k d_k$\;
   
    \nl Next randomly generate a set of new samples $s_k$ of cardinality $|s_k|$\;
    
    \nl Define $S_k \overset{\Delta}{=}  S_{k-1} \cup s_k$\;
    
    \nl Construct $f_k(\alpha)=\frac{1}{2}\alpha^T Q^k \alpha + \frac{1}{|S_k|} \sum_{i=1}^{|S_k|} \max\{0,1-w_i\langle \alpha, Q^k_i \rangle \}$\;
    
    \nl Randomly generate a set of new samples $T_k$ of cardinality $|T_k|=|S_k|$ independent of $S_{k}$\;
    
    \nl Construct $\hat{f}_k(\alpha)=\frac{1}{2}\alpha^{\top} \hat{Q}^k \alpha + \frac{1}{|T_k|} \sum_{i=1}^{|T_k|} \max\{0,1-\hat{w}_i\langle \alpha, \hat{Q}^k_i \rangle \}$ based on $T_k$\;\label{SAA2}

    \nl Let $\beta_k=(\alpha_k,\overrightarrow{0}) \in \mathbb{R}^{|S_k|}$, $\hat{\beta}_{k-1}=(\hat{\alpha}_{k-1},\overrightarrow{0}) \in \mathbb{R}^{|S_k|}$ and $d_k=(d_k,\overrightarrow{0}) \in \mathbb{R}^{|S_k|}$\;
    
    \eIf{$f_k(\beta_k)-f_k(\hat{\beta}_{k-1}) \leq \eta_1(\hat{f}_{k}(\beta_k)-\hat{f}_{k}(\hat{\beta}_{k-1}))$ and $||d_k||>\eta_2\delta_{k-1}$}{\nl
        $\hat{\alpha}_k \leftarrow \beta_k, \quad \delta_k \leftarrow \min \{\gamma\delta_{k-1},\delta_{max} \}$\;
    }{\nl
            $\hat{\alpha}_k \leftarrow \hat{\beta}_{k-1}, \quad \delta_k \leftarrow \max\{\frac{\delta_{k-1}}{\gamma}, \delta_{min}\}$\;
        } 

}
\end{algorithm}
\textbf{Remark 1}: (i) The initial choice of step size $\delta_0$ may require some tuning depending on the data set but it does not affect the convergence property of this algorithm.\\ 
(ii) According to Theorem \ref{delta 1}, $\delta_{max} := \frac{\varepsilon}{\zeta}$ and 
\begin{equation*}
    \zeta = \max \Bigl\{4n\kappa,\eta_2,\frac{4\kappa}{C_1(1-\eta_1)}\Bigr\},
\end{equation*}
 where $C_1=\frac{m_1}{2n}$, $1< n \in \mathbb{Z}^+$.\\
 (iii) Note that replacing $d_k$ with an arbitrary subgradient may not ensure descent unless the function $f_{k-1}$ is smooth. Thus, simply substituting $d_k$ with $g_k$ and using SGD-based method may not guarantee the convergence rate shown in Theorem 
 \ref{Convergence rate}.\\
 (iv) In general, optimization problems are stated with the decision variables in a fixed dimension. However, since we have designed an adaptive sampling approach, the dimension of our decision variables grows iteratively according to the increase in sample size.  
 \\
 
\setcounter{AlgoLine}{0}
\begin{algorithm}[] 
\caption{Line Search Algorithm}
\scriptsize
\label{LSA}

\nl Set $\frac{1}{4}\leq m_1<0.5$, $\frac{1}{4} \leq m_2<m_1$, $1< n \in \mathbb{Z}^+$ and $b=\frac{\delta_k}{n}$.\;

\nl Choose $t||d_k|| \in [b,\delta_k]$.\; 
    
    \If{$t \in L \backslash R$}{\nl $t=2t$ until $t \in R$ or $t||d_k|| > \delta_k$. Set $I=[t/2,t]$.\;
    
    \eIf{$t ||d_k|| > \delta_k$}{\nl Return $t/2$.}{
    \eIf{$t \in L \cap R$}{ \nl Return $t$}{
        \While{$t \notin L \cap R$}{ 
        \nl Set $t$ be the middle point of $I$\;\label{sw}
    
        \eIf{$ t \in R \backslash L $}
        {\nl Replace I by its left half}
        {\nl Replace I by its right half} \label{ew}
        }
    }
    }
    }
    
    \If{$t \in R\backslash L$}{\nl $t=t/2$ until $t \in L$ or $t||d_k||<b$. Set I be the interval of $[t,2t]$.\;
    
    \eIf{$t \in L \cap R$}{\nl Return $t$}
    {\eIf{$t||d_k||<b$}
    {\nl Return $t=0$}{\nl Repeat line (\ref{sw})-(\ref{ew}) while $t \notin L \cap R$. 
    }}
    }
\nl Return $t$
\end{algorithm}

\textbf{Remark 2}: Algorithm \ref{LSA} will have two possible outcomes. One is it returns a suitable step size $t$ such that $t \in L \cap R$ and $t ||d_k|| \geq \frac{\delta_k}{n}$. The other is it keeps shrinking the step size until $t ||d_k|| < \frac{\delta_k}{n}$. This case is captured in step 11 of Algorithm 2 and the algorithm will return $t = 0$. Since $\alpha_k = \hat{\alpha}_{k-1} + t_k d_k$, $t=0$ means the decision variable will not be updated in this iteration. In the next iteration, Algorithm \ref{SCSQL-KSVM} will generate more samples and these new samples may improve the search direction. This whole procedure creates a subroutine to first terminate the line-search algorithm and then improve the search direction by further sampling. This phenomenon is similar to retaining the previous incumbent solution in the stochastic decomposition algorithm~\cite{HS1994}.

\section{Convergence Properties}
\label{CA}
While the SCS algorithm and the stochastic trust region method share some similarities in their proofs of convergence, the specifics of these arguments are different because the former is based on direction-finding and the latter is based on trust regions. For example, the fully-linear property in stochastic trust region method plays no role in the convergence proof of this paper. Similarly, the sufficient decrease condition (Lemma \ref{d2}) in our algorithm relies on Wolfe's line search algorithm while the condition used for a stochastic trust region method requires the Cauchy decrease property.   
In the following subsections, Theorem \ref{L^0} shows the sample complexity of the algorithm, Theorem \ref{delta 1} and the subsequent corollary show that sufficient descent can be obtained and as a result, the stochastic sequence is a supermartingale. 
In \S \ref{4.3}, we apply the renewal reward theorem and optional stopping time theorem \cite{GGD2020} to establish the  convergence rate of the SCS algorithm. The specific results are provided in Theorem \ref{Convergence rate}. 

\underline{Notational Guide.} We will adopt the convention that random variables will be denoted by capital letter (e.g., $B$) while a realization will be denoted by the corresponding lower case letter (e.g., $b$). Also, to avoid any notational confusion, please note that $\mathcal{B}$ will denote a ball and $C_i$'s are ordinary constants.

\subsection{Stochastic Local Approximations} \label{4.1}
In this subsection, we aim to demonstrate that with an adequate number of samples, the function $f_k$ will serve as an ``effective'' stochastic local approximation of $f$. The criteria for what constitute ``effective" in this context is provided in Definition \ref{kappa approximation}, which provides a precise metric. Additionally, Definition \ref{stocahstic accuracy kappa approximation} extends this concept to a stochastic framework. 

\begin{definition} \label{kappa approximation}  A function $f_k$ is a local $\kappa$-approximation of $f$ on $\mathcal{B}(\hat{\alpha}_k,\delta_k)$, if for all $\alpha \in \mathcal{B}(\hat{\alpha}_k,\delta_k)$,
\begin{equation*}
|f(\alpha)-f_k(\alpha)| \leq \kappa \delta_k^2.
\end{equation*}

\end{definition}


\begin{definition} \label{stocahstic kappa approximation}Let $\Delta_k$, $\hat{A}_k$ denote the random counterpart of $\delta_k$,  $\hat{\alpha}_k$, respectively, and $\mathcal{F}_{k}$ denote the $\sigma$-algebra generated by $ \{\mathcal{F}_{i},\hat{F}_i,
\hat{F}_{i+1}\}_{i=0}^{k-1}$. Notice that this definition allows us to include the value of the so-called incumbent sequences (of Stochastic Decomposition \cite {HS1994}) as part of the $\sigma$-algebra. A sequence of random models $\{\mathcal{F}_{k}\}$ is said to be $\mu_1$-probabilistically $\kappa$-approximation of $f$ with respect to $\{\mathcal{B}(\hat{A}_k,\Delta_k)\}$ if the event 
\begin{equation*}
    I_k \overset{\Delta}{=}  \{ \mathcal{F}_{k} \ \textit{is} \ \kappa- \textit{approximation of}  \ f \textit{on} \ \mathcal{B}(\hat{A}_k,\Delta_k)\}
\end{equation*}

\noindent satisfies the condition 
\begin{equation*}
    \mathbb{P}(I_k|\mathcal{F}_{k-1}) \geq \mu_1 > 0.
\end{equation*}
\end{definition}


\noindent \textbf{Remark 2}: (i) Unlike \cite{PS2020} which imposes a generic set of guidelines for stochastic line-search methods including requirements on the gradient approximation, our setup accommodates non-smoothness of the kernel SVM problems. For this reason, our requirements simply focus on sufficiently accurate function estimation without invoking gradient approximations.
(ii) In the interest of clarity, we emphasize that the incumbent decision $\hat{\alpha}_k$ has its random counterpart denoted as $\hat{A}_k$.\\

\noindent Given the relation $\beta = \sum_{i=i}^{|S_k|}\alpha_i\phi(z_i)$, we can define the function $f(\alpha)$ as a transformation of $g(\beta)$
\begin{equation}
    g(\beta)=f(\alpha):=\frac{1}{2}\langle \alpha, Q^k \alpha \rangle +  E[\max\{0,1-W \langle \alpha, Q_i^k(Z) \rangle\}],
\end{equation}

\noindent where $Q_i^k(Z) = [\langle \phi(z_1), \phi(Z) \rangle, \langle \phi(z_2), \phi(Z) \rangle,...,\langle \phi(z_{|S_k|}), \phi(Z) \rangle]$. Similarly, we can also express $f_k(\alpha)$ in terms of $g_k(\beta)$ as follows.

\begin{equation}
    f_k(\alpha)=g_k(\beta):=\frac{1}{2}||\beta||^2 + \frac{1}{|S_k|}\sum_{i=1}^{|S_k|} \max\{0,1-w_i \langle \beta, \phi(z_i) \rangle\}.
\end{equation}

\begin{theorem} \label{L^0} 
Assume that $f$ and $f_k$ are Lipschitz continuous with constant $L_f$, for any $0 < \varepsilon <1$, $\kappa > \frac{4L_f}{\delta_{min}}$. If for all $i$, we have $|\langle \beta, \phi(z_i)\rangle| \leq M$ and 
\begin{equation} \label{sample_size}
 |S_k| \geq -8\log(\varepsilon/2)\cdot \frac{(M+1)^2}{\kappa^2\delta_k^4},    
\end{equation}
then
\begin{equation*}
    \mathbb{P}(|f_k(\alpha)-f(\alpha)| \leq \kappa \delta_k^2, \ \forall \alpha \in \mathcal{B}(\hat{\alpha}_k,\delta_k) ) \geq 1-\varepsilon.
\end{equation*}
\end{theorem}

\begin{proof}
First, for the point $\hat{\alpha}_k$, given that $\hat{\beta}_k = \sum_{i=i}^{|S_k|}\hat{\alpha}_{k,i}\phi(z_i)$, we have 
\begin{equation*}
    \begin{aligned}
        & f(\hat{\alpha}_k)=g(\hat{\beta}_k)=\frac{1}{2}||\hat{\beta}_k||^2+E [\max\{0,1-W \langle \hat{\beta}_k, \phi(Z) \rangle\}],\\
        & f_k(\hat{\alpha}_k)=g_k(\hat{\beta}_k)=
        \frac{1}{2}||\hat{\beta}_k||^2 + \frac{1}{|S_k|}\sum_{i=1}^{|S_k|} \max\{0,1-w_i \langle \hat{\beta}_k, \phi(z_i) \rangle\}.
    \end{aligned}
\end{equation*}

\noindent Thus, 

\begin{equation*}
    \begin{aligned}
        f_k(\hat{\alpha}_k)-f(\hat{\alpha}_k) & =g_k(\hat{\beta}_k)-g(\hat{\beta}_k)\\
        & =\frac{1}{|S_k|}\sum_{i=1}^{|S_k|} \max\{0,1-w_i \langle \hat{\beta}_k, \phi(z_i) \rangle\}-E [\max\{0,1-W \langle \hat{\beta}_k, \phi(Z) \rangle\}].
    \end{aligned}
\end{equation*}
\noindent Using the assumption that $|\langle \beta, \phi(z_i)\rangle| \leq M$ for all $i$, Hoeffding's inequality implies that
\begin{equation*}
    \mathbb{P}(|f_k(\hat{\alpha}_k)-f(\hat{\alpha}_k)| \leq \frac{1}{2}\kappa \delta_k^2) \geq 1-2\exp \left (-\frac{\kappa^2\delta_k^4 |S_k|^2}{8 \cdot |S_k| \cdot (M+1)^2} \right ) \geq 1-\varepsilon,
\end{equation*}
\noindent which indicates that when 
\begin{displaymath}
   |S_k| \geq -8\log(\varepsilon/2)\cdot \frac{(M+1)^2}{\kappa^2\delta_k^4}, 
\end{displaymath} 
we have 
\begin{equation*}
    \mathbb{P}(|f_k(\hat{\alpha}_k)-f(\hat{\alpha}_k)| \leq \frac{1}{2}\kappa\delta_k^2 ) \geq 1-\varepsilon.
\end{equation*}

\noindent For any other $\alpha \in \mathcal{B}(\hat{\alpha}_k,\delta_k)$, if $|f_k(\hat{\alpha}_k)-f(\hat{\alpha}_k)| \leq \frac{1}{2}\kappa\delta_k^2$, then

\begin{equation*}
    \begin{aligned}
        |f_k(\alpha)-f(\alpha)| & \leq |f_k(\alpha)-f_k(\alpha_k)|+|f_k(\alpha_k)-f(\alpha_k)|+ |f(\alpha_k)-f(\alpha)| \\
        & \leq 2L_f \cdot \delta_k + \frac{1}{2} \kappa \delta_k^2\\
        & \leq \kappa \delta_k^2,
    \end{aligned}
\end{equation*}

\noindent where the second last inequality is due to the assumption that $f_k$ and $f$ are Lipschitz continuous and the last inequality is because $\kappa>\frac{4L_f}{\delta_{min}}>\frac{4L_f}{\delta_k}$. Thus, we conclude that if $|S_k|$ satisfies Equation \eqref{sample_size}, then 

\begin{equation*} \label{fl1}
    \mathbb{P}(|f_k(\alpha)-f(\alpha)| \leq \kappa \delta_k^2, \ \forall \alpha \in \mathcal{B}(\hat{\alpha}_k,\delta_k) ) \geq 1-\varepsilon.
\end{equation*}
\end{proof}

\subsection{Sufficient Decreases and Supermartingale Property} 
This section establishes the stochastic variant of the sufficient decrease condition, known as the supermartingale property, in Theorem \ref{delta 1}. The proof is segmented into four distinct cases, each based on whether the conditions of $\kappa$-approximation (as defined in Definition \ref{stocahstic kappa approximation}) and $\varepsilon_F$-accuracy (as outlined in Definition \ref{stocahstic accuracy}) are met or not. Lemmas \ref{d1} - \ref{d3} detail the outcomes for these various situations. Following these lemmas, the supermartingale property and its comprehensive proof are presented.

\begin{definition} The value estimates $\hat{f}_k \overset{\Delta}{=} \hat{f}_k(\hat{\alpha}_k)$ and $\hat{f}_{k+1} \overset{\Delta}{=}  \hat{f}_k(\hat{\alpha}_k+td_k)$ are $\varepsilon_F$-accurate estimates of $f(\hat{\alpha}_k)$ and $f(\hat{\alpha}_k+td_k)$, respectively, for a given $\delta_k$ if and only if 
\begin{equation*}
    |\hat{f}_k-f(\hat{\alpha}_k)| \leq \varepsilon_F \delta_k^2, \quad |\hat{f}_{k+1}-f(\hat{\alpha}_k+td_k)| \leq \varepsilon_F \delta_k^2.
\end{equation*}
\end{definition}

\begin{definition} \label{stocahstic accuracy} A sequence of model estimates $\{\hat{F}_k,\hat{F}_{k+1}\}$ is said to be $\mu_2$-probabilistically $\varepsilon_F$-accurate with respect to $\{\hat{A}_k,\Delta_k,S_k\}$ if the event 
\begin{equation}
    J_k \overset{\Delta}{=}  \{ \hat{F}_k,\hat{F}_{k+1} \ \textit{are} \ \varepsilon_F\textit{-accurate estimates }  \ \textit{for} \ \Delta_k \}
\end{equation}
\noindent satisfies the condition 
\begin{equation}
    \mathbb{P}(J_k|\mathcal{F}_{k-1}) \geq \mu_2 >0.
\end{equation}
\end{definition}

\begin{lemma} \label{d1}
Suppose that ${f}_k$ is a $\kappa$-approximation of $f$ on $\mathcal{B}(\hat{\alpha}_k ,\delta_k)$ and $C_1 = \frac{1}{8n}$, where $n$ is specified in Algorithm \ref{LSA}. If 
    \begin{equation} \label{d1c}
    \delta_k \leq \frac{1}{16n\kappa} ||d_k||
    \end{equation}

    \noindent and Algorithm \ref{LSA} returns a suitable step size $t \in L$ and $t ||d_k|| > \frac{\delta_k}{n}$, then we have 
    \begin{equation} \label{d1con}
        f(\hat{\alpha}_k+t d_k)-f(\hat{\alpha}_k) \leq -C_1 ||d_k|| \delta_k.
    \end{equation}
\end{lemma}
\begin{proof}
Since $t \in L$ and $t ||d_k|| > \frac{\delta_k}{n}$, we have 
\begin{equation*}
    f_k(\hat{\alpha}_k+t d_k)-f_k(\hat{\alpha}_k) \leq -m_1||d_k||^2 t \leq -\frac{m_1}{n}||d_k||\delta_k.
\end{equation*}
\noindent Also, since $\hat{f}_k$ is $\kappa-approximation$ of $f$ on $\mathcal{B}(\hat{\alpha}_k,\delta_k)$, we have 
\begin{equation}
    \begin{aligned}
          \quad \ \ f(\hat{\alpha}_k+t d_k)-f(\hat{\alpha}_k) 
         & = f(\hat{\alpha}_k+t d_k)-f_k(\hat{\alpha}_k+t d_k)+f_k(\hat{\alpha}_k+t d_k)\\
         &-f_k(\hat{\alpha}_k)+f_k(\hat{\alpha}_k)-f(\hat{\alpha}_k)\\
         & \leq 2\kappa\delta_k^2-\frac{m_1}{n}||d_k||\delta_k \\
         & \leq -C_1||d_k||\delta_k,
    \end{aligned} 
\end{equation}
which concludes the proof.
\end{proof}

\begin{lemma} \label{d2} Suppose that $f_k$ is $\kappa$-approximation of $f$ on the ball $\mathcal{B}(\hat{\alpha}_k,\delta_k)$ and the estimates $(\hat{f}_k,\hat{f}_{k+1})$ are $\varepsilon_F$-accurate with $\varepsilon_F \leq \kappa$. Moreover, if 
\begin{equation} \label{d2c}
    \delta_k \leq \min \{\frac{(1-\eta_1)C_1 ||d_k||}{2\kappa}, \frac{||d_k||}{\eta_2}\},
\end{equation}

\noindent and Algorithm \ref{LSA} returns a suitable step size $t \in L$ and $t ||d_k|| > \frac{\delta_k}{n}$, then the $k$-th iteration will be successful (Update $\hat{\alpha}_k \leftarrow \beta_k$, and $ \delta_k \leftarrow \min \{\gamma\delta_{k-1},\delta_{max} \}$).
\end{lemma}

\begin{proof}
    Since $t \in L$ and $t ||d_k|| > \frac{\delta_k}{n}$, we have
    \begin{equation*}
    f_k(\hat{\alpha}_k+t d_k)-f_k(\hat{\alpha}_k)  \leq -\frac{m_1}{n} ||d_k|| \delta_k \leq -2C_1 ||d_k|| \delta_k,
    \end{equation*}
    
    \noindent where the last inequality holds because $m_1 \geq \frac{1}{4}$. Also, $f_k$ is $\kappa$-approximation of $f$ on $\mathcal{B}(\hat{\alpha}_k,\delta_k)$ and the estimates $(\hat{f}_{k},\hat{f}_{k+1})$ are $\varepsilon_F$-accurate with $\varepsilon_F \leq \kappa$ implies that 
    
    \begin{equation*}
        \begin{aligned}
            \rho_k & = \frac{\hat{f}_{k+1}-\hat{f}_{k}}{f_k(\hat{\alpha}_k-t d_k)-f_k(\hat{\alpha}_k)}\\
            & = \frac{\hat{f}_{k+1}-f(\hat{\alpha}_k+t d_k)}{f_k(\hat{\alpha}_k+t d_k)-f_k(\hat{\alpha}_k)}+\frac{f(\hat{\alpha}_k+t d_k)-f_k(\hat{\alpha}_k+t d_k)}{f_k(\hat{\alpha}_k+t d_k)-f_k(\hat{\alpha}_k)}+\frac{f_k(\hat{\alpha}_k+t d_k)-f_k(\hat{\alpha}_k)}{f_k(\hat{\alpha}_k+t d_k)-f_k(\hat{\alpha}_k)}\\
            & + \frac{f_k(\hat{\alpha}_k)-f(\hat{\alpha}_k)}{f_k(\hat{\alpha}_k+t d_k)-f_k(\hat{\alpha}_k)}+\frac{f(\hat{\alpha}_k)-\hat{f}_{k}}{f_k(\hat{\alpha}_k+t d_k)-f_k(\hat{\alpha}_k)},\\
        \end{aligned}
    \end{equation*}
    
\noindent which indicates that
\begin{equation*}
    1-\rho_k \leq \frac{ 2\kappa\delta_k}{C_1||d_k||}\leq 1-\eta_1.
\end{equation*}

\noindent Thus, we have $\rho_k \geq \eta_1$, $||d_k||>\eta_2\delta_k$ and the iteration is successful. 
\end{proof}

\begin{lemma} \label{d3} Suppose the function value estimates $\{(\hat{f}_{k},\hat{f}_{k+1})\}$ are $\varepsilon_F$-accurate and
\begin{equation*} \label{d3c}
    \varepsilon_F < \eta_1 \eta_2 C_1.
\end{equation*}
\noindent If the $k$th iteration is successful, then the improvement in $f$ is bounded such that 

\begin{equation*}
    f(\hat{\alpha}^{k+1})-f(\hat{\alpha}_k) \leq -C_2\delta_k^2,
\end{equation*}

\noindent where $C_2 \overset{\Delta}{=} 2\eta_1 \eta_2 C_1-2\varepsilon_F$.
\end{lemma}

\begin{proof}
    Since $t \in L$ and $t ||d_k|| > \frac{\delta_k}{n}$, we have
    \begin{equation*}
    f_k(\hat{\alpha}_k+t d_k)-f_k(\hat{\alpha}_k) \leq -2C_1 ||d_k|| \delta_k,
    \end{equation*}
    Also, since the iteration is successful, we have $||d_k||>\eta_2 \delta_k$. Thus, we have
    
    \begin{equation*}
        \hat{f}_{k}-\hat{f}_{k+1} \geq \eta_1(f(\hat{\alpha}_k)-f(\hat{\alpha}_k+td_k)) \geq 2\eta_1 C_1 ||d_k||\delta_k \geq 2\eta_1 \eta_2 C_1 \delta_k^2.
     \end{equation*}
    Then, since the estimates are $\varepsilon_F$ -accurate, we have that the improvement in $f$ can be bounded as  
    \begin{displaymath}
        f(\hat{\alpha}_k+t d_k)-f(\hat{\alpha}_k) \leq f(\hat{\alpha}_k+t d_k)-\hat{f}_{k+1}+\hat{f}_{k+1}-\hat{f}_{k}+\hat{f}_{k}-f(\hat{\alpha}_k) \leq -C_2 \delta_k^2. 
    \end{displaymath}
\end{proof}

\begin{theorem} \label{delta 1} Let the random function $V_k \overset{\Delta}{=} F_{k+1}-F_{k}$, the corresponding realization be $v_k$, 

\begin{equation} \label{xi}
    \zeta = \max\{4n\kappa,\eta_2,\frac{4\kappa}{C_1(1-\eta_1)}\}  \quad and \quad  \mu_1\mu_2 > \frac{1}{2}.
\end{equation}

\noindent  Then for any $\varepsilon>0$ such that $||d_k|| \geq \varepsilon$ and $\zeta \Delta_k \leq \varepsilon$, we have 
\begin{equation*}
    \mathbb{E}[V_k \big| \ ||d_k|| \geq \varepsilon, \zeta \Delta_k \leq \varepsilon] \leq -\frac{1}{2}C_1 ||d_k|| \Delta_k \leq -\theta \varepsilon \Delta_k,
\end{equation*}
\noindent where $\theta = \frac{1}{2}  C_1$.
\end{theorem} 

\begin{proof}
First of all, if $k$th iteration is successful, i.e. $\hat{\alpha}_{k+1}=\hat{\alpha}_k-td_k$ we have 
\begin{equation} \label{s}
    v_k = f(\hat{\alpha}_k+td_k)-f(\hat{\alpha}_k).
\end{equation}

\noindent If $k$th iteration is unsuccessful, i.e. $\hat{\alpha}_{k+1}=\hat{\alpha}_k$ we have

\begin{equation} \label{b1}
    v_k = f(\hat{\alpha}_{k})-f(\hat{\alpha}_k)=0.
\end{equation}

\noindent Then we will divide the analysis into 4 cases according to the states (true/false) observed for the pair $(I_k, J_k)$.

\noindent (a) $I_k$ and $J_k$ are both true. Since the $f_k$ is a $\kappa$-approximation of $f$ on $\mathcal{B}(\hat{\alpha}_k ,\delta_k)$ and condition (\ref{d1c}) is satisfied, we have Lemma \ref{d1} holds. Also, since the estimates $(\hat{f}_k,\hat{f}_{k+1})$ are $\varepsilon_F$-accurate and condition (\ref{d2c}) is satisfied, we have Lemma \ref{d2} holds.  Combining (\ref{d1con}) with (\ref{s}), we have 

\begin{equation} \label{b2}
    v_k \leq  -C_1 ||d_k||\delta_k \overset{\Delta}{=} b_1.
\end{equation}

\noindent (b) $I_k$ is true but $J_k$ is false. Since $f_k$ is a $\kappa$-approximation of $f$ on $\mathcal{B}(\hat{\alpha}_k ,\delta_k) \cap X$ and condition (\ref{d1c}) is satisfied, it follows that Lemma \ref{d1} still holds. If the iteration is successful, we have (\ref{b2}), otherwise we have (\ref{b1}). Thus, we have
$v_k \leq 0$.

\noindent (c) $I_k$ is false but $J_k$ is true. If the iteration is successful, since the estimates $(\hat{f}_k,\hat{f}_{k+1})$ are $\varepsilon_F$-accurate and condition (\ref{d3c}) is satisfied, Lemma \ref{d3} holds. Hence, 
\begin{equation*}
    v_k \leq -C_2 \delta_k^2.
\end{equation*}

\noindent If the iteration is unsuccessful, we have (\ref{b1}). Thus, we have $v_k \leq 0$ whether the iteration is successful or not.

\noindent (d) $I_k$ and $J_k$ are both false. Since $f$ is convex and $t||d_k||<\delta_k$, for any $g(t) \in \partial f(\hat{\alpha}_k+td_k)$, with  the assumption in \S \ref{conditions} which implies that $||g(t)||$ is bounded, we have 

\begin{equation*}
    v_k=f(\hat{\alpha}_k+td_k)-f(\hat{\alpha}_k) \leq g(t)^T td_k \leq C_3\delta_k.
\end{equation*}

\noindent If the iteration is successful, then
\begin{equation*}
    v_k \leq C_3 \delta_k \overset{\Delta}{=} b_2 .
\end{equation*}

\noindent If the iteration is unsuccessful, we have (\ref{b1}). Thus, we have $v_k \leq  b_2$ whether the iteration is successful or not.

With the above four cases, we can bound $\mathbb{E}[V_k \big| \ ||d_k|| \geq \varepsilon, \zeta \Delta_k \leq \varepsilon]$ based on different outcomes of $I_k$ and $J_k$. Let $B_1$ and $B_2$ be the random counterparts of $b_1$ and $b_2$.  Then we have 
\begin{equation*}
    \begin{aligned}
         & \quad \  \mathbb{E}[V_k \big| \ ||d_k|| \geq \varepsilon, \zeta \Delta_k \leq \varepsilon] \\
         & \leq \mu_1 \mu_2 B_1 + (\mu_1(1-\mu_2)+\mu_2(1-\mu_1))\cdot 0 + (1-\mu_1)(1-\mu_2) B_2 \\
         & = \mu_1 \mu_2  (- C_1 ||d_k||\Delta_k) + (1-\mu_1)(1-\mu_2)C_3 \Delta_k. \\
    \end{aligned}
\end{equation*}

\noindent Choose $\mu_1 \in (1/2,1)$ and $\mu_2 \in (1/2,1)$ large enough such that 

\begin{equation*}
 \mu_1 \mu_2 \geq 1/2 \quad \text{and \quad }\frac{\mu_1 \mu_2}{(1-\mu_1)(1-\mu_2)} \geq \frac{2 C_3}{c_1||d_k||},
\end{equation*}

\noindent we have
\begin{displaymath}
    \begin{aligned}
    \mathbb{E}[V_k \big| \ ||d_k|| \geq \varepsilon, \zeta \Delta_k \leq \varepsilon] & \leq - \frac{1}{2} C_1 ||d_k||\Delta_k. 
    \end{aligned}
\end{displaymath}
\end{proof}

\noindent A quick observation of the concluding condition reveals that the requirement $\Delta_k \leq \frac{\varepsilon}{\zeta}$ implies that the supermartingale property holds. This prompts us to impose the condition that restricts $\delta_{max} = \frac{\varepsilon}{\zeta}$. This property is formalized in the following corollary. 
\begin{corollary} Let 
\begin{equation*}
    T_{\varepsilon} = \inf \{ k \geq 0: ||d_k|| < \varepsilon\}.
\end{equation*}

\noindent Then $T_{\varepsilon}$ is a stopping time for the stochastic process $\hat{A}^k$. Moreover, conditioned on $T_{\varepsilon} \geq k$, $F_k$ is a supermartingale. 

\end{corollary}

\begin{proof}

{    
    From Theorem \ref{delta 1}, we have
    \begin{equation} \label{super 1}
        \mathbb{E}[F_k | \mathcal{F}_{k-1}, T_{\varepsilon}  > k] \leq F_{k-1} - \Theta \varepsilon \Delta_k.
    \end{equation}
}

\noindent Hence, $F_k$ is a supermartingale. \end{proof}

\subsection{Convergence Rate} \label{4.3}

Building upon the results established in Theorem \ref{delta 1}, where $f_k$ is demonstrated to be a supermartingale and $T_{\varepsilon}$ is identified as a stopping time, we proceed to construct a renewal reward process for analyzing the bound on the expected value of $T_{\varepsilon}$.
As highlighted in the abstract, Theorem \ref{Convergence rate} confirms the rate of convergence to be $O(1/\varepsilon^2)$.
To begin with, let us define the renewal process $\{ A_l \}$ as follows: set $A_0 = 0$, and for each $l > 0$, define $A_l = \inf \{ m > A_{l-1} : \zeta \Delta_m \geq \varepsilon \}$, with $\zeta$ being specified in \eqref{xi}. Additionally, we define the inter-arrival times $\tau_l = A_l - A_{l-1}$. Lastly, we introduce the counting process $N(k) = \max \{n: A_l \leq k\}$, representing the number of renewals occurring up to the $k^{th}$ iteration.

\begin{lemma} \label{tau} Let $ \frac{1}{2} < p=\mu_1 \mu_2 \leq \mathbb{P}(I_k \cap J_k)$. Then for all $l \geq 1$,

\begin{equation*}
    \mathbb{E}[\tau_l] \leq \frac{p}{2p-1}.
\end{equation*}

\begin{proof}
First,
    \begin{equation*}
        \begin{aligned}
            \mathbb{E}[\tau_l] & = \mathbb{E}[\tau_l | \zeta \Delta_{A_{l-1}} > \varepsilon] \cdot \mathbb{P}(\zeta \Delta_{A_{l-1}} > \varepsilon) +  \mathbb{E}[\tau_l | \zeta \Delta_{A_{l-1}} = \varepsilon] \cdot \mathbb{P}(\zeta \Delta_{A_{l-1}} = \varepsilon)\\
            & \leq \max \{ \mathbb{E}[\tau_l | \zeta \Delta_{A_{l-1}} > \varepsilon], \mathbb{E}[\tau_n | \zeta \Delta_{A_{l-1}} = \varepsilon] \}.
        \end{aligned}
    \end{equation*}
    
\noindent If $\zeta \Delta_{A_{l-1}} > \varepsilon$, according to algorithm \ref{SCSQL-KSVM}, 

\begin{equation} \label{tau 1}
    \mathbb{E}[\tau_l | \zeta \Delta_{A_{l-1}} > \varepsilon] = 1.
\end{equation}

\noindent If $\zeta \Delta_{A_{l-1}} = \varepsilon$, by Theorem \ref{delta 1}, we can treat $\{ \Delta_{A_{l-1}},...,\Delta_{A_{l}}\}$ as a random walk, and we have 

\begin{equation} \label{tau 2}
    \mathbb{E}[\tau_l | \zeta \Delta_{A_{l-1}} = \varepsilon] \leq \frac{p}{2p-1}.
\end{equation}

\noindent Combining (\ref{tau 1}) and (\ref{tau 2}) completes the proof. \end{proof}

\end{lemma}

\begin{lemma} \label{N(T)} Let $\zeta$ and $\theta$ be the same as in Theorem \ref{delta 1} and $\Delta_{max} = \frac{\varepsilon}{\zeta}$, then
\begin{equation*}
    \mathbb{E}[N(T_{\varepsilon})] \leq \frac{2 \zeta F_{max} }{\theta \varepsilon^2}+\frac{\zeta \Delta_{max}}{\varepsilon} = \frac{2 \zeta F_{max} }{\theta \varepsilon^2}+1,
\end{equation*}

where $F_{max}=max_{i \leq (k \wedge T_{\varepsilon})}|F_i|$.
\end{lemma}

\begin{proof}
    We will first show that 
    \begin{equation} \label{R_k}
        R_{k \wedge T_{\varepsilon}} =F_{k \wedge T_{\varepsilon}} + \Theta \varepsilon \sum_{j=0}^{k \wedge T_{\varepsilon}} \Delta_j
    \end{equation}
\noindent is a supermartingale. Using (\ref{super 1}),

\begin{equation} \label{Supermar. proof}
    \begin{aligned}
        \mathbb{E}[R_{k \wedge T_{\varepsilon}}|\mathcal{F}_{k-1}] & = \mathbb{E}[F_{k \wedge T_{\varepsilon}}|\mathcal{F}_{k-1}] + \mathbb{E}[\Theta \varepsilon \sum_{j=0}^{k \wedge T_{\varepsilon}} \Delta_j | \mathcal{F}_{k-1} ]\\
        & \leq F_{k-1} - \Theta \varepsilon \Delta_k + \mathbb{E}[\Theta \varepsilon \sum_{j=0}^{k \wedge T_{\varepsilon}} \Delta_j | \mathcal{F}_{k-1} ] \\
        & = F_{k-1} + \mathbb{E}[\Theta \varepsilon \sum_{j=0}^{(k-1) \wedge T_{\varepsilon}} \Delta_j | \mathcal{F}_{k-1} ] \\
        & = F_{k-1} + \Theta \varepsilon \sum_{j=0}^{(k-1) \wedge T_{\varepsilon}} \Delta_j \\
        & = R_{(k-1) \wedge T_{\varepsilon}},
    \end{aligned}
\end{equation}

\noindent where the summation in the last expectation in \eqref{Supermar. proof} is true by moving $\Theta \varepsilon \Delta_k$ inside the summation so that it has one less term if $k < T_{\varepsilon}$.

\noindent If $k < T_{\varepsilon}$, then 

\begin{equation*}
    |R_{k \wedge T_{\varepsilon}}| = |R_k| \leq F_{max} + \Theta \varepsilon k \Delta_{max}.
\end{equation*}

\noindent If $k \geq T_{\varepsilon}$, then 

\begin{equation*}
    |R_{k \wedge T_{\varepsilon}}| = |R_{\varepsilon}| \leq  F_{max} + \Theta \varepsilon T_{\varepsilon} \Delta_{max}.
\end{equation*}

\noindent This is also bounded almost surely since $T_{\varepsilon}$ is bounded almost surely. Hence, according to \eqref{R_k} and the optional stopping theorem \cite{GGD2020}, we have

\begin{equation} \label{optinal 1}
    \mathbb{E}[\Theta \varepsilon \sum_{j=0}^{ T_{\varepsilon}} \Delta_j] \leq \mathbb{E}[R_{T_{\varepsilon}}] + F_{max} \leq \mathbb{E}[R_0] + F_{max} \leq 2 F_{max} + \Theta \varepsilon \Delta_{max}.
\end{equation}

\noindent Furthermore, since the renewal $A_n$ happens when $\zeta \Delta_j \geq \varepsilon$ and $N(T_{\varepsilon})$ is a subset of $\{0,1,2,...,T_{\varepsilon}\}$, we have 

\begin{equation} \label{optinal 2}
    \Theta \varepsilon \Big(\sum_{j=0}^{ T_{\varepsilon}} \zeta \Delta_j\Big) \geq \Theta \varepsilon \Big( N(T_{\varepsilon}) \varepsilon \Big).
\end{equation}

\noindent Combining (\ref{optinal 1}) and (\ref{optinal 2}), we have \begin{displaymath}
    \mathbb{E}[N(T_{\varepsilon})] \leq \frac{2 \zeta F_{max} + \zeta \Theta \varepsilon \Delta_{max}}{\Theta\varepsilon^2} \leq \frac{2 \zeta F_{max} }{\Theta \varepsilon^2}+\frac{\zeta \Delta_{max}}{\varepsilon}.
    \end{displaymath}  
\end{proof}

\begin{theorem}\label{Convergence rate}
Under conditions enunciated in \S \ref{conditions}, we have
    \begin{equation} \label{convergence equation}
    \mathbb{E}[T_{\varepsilon}] \leq \frac{p}{2p-1}\Big(\frac{2 \zeta F_{max} }{\Theta \varepsilon^2}+2\Big).
\end{equation}
\end{theorem}

\begin{proof}
First, note that $N(T_{\varepsilon})+1$ is a stopping time for the renewal process $\{ A_n: n \geq 0\}$. Thus, using Wald's equation (inequality form) \cite{GGD2020}, we have 
    
\begin{equation*}
        \mathbb{E}[A_{N(T_{\varepsilon})+1}] \leq \mathbb{E}[\tau_1] \mathbb{E}[N(T_{\varepsilon})+1].
\end{equation*}
    
\noindent Moreover, since $A_{N(T_{\varepsilon})+1} \geq T_{\varepsilon}$, we have

\begin{equation*}
    \mathbb{E}[T_{\varepsilon}] \leq \mathbb{E}[\tau_1] \mathbb{E}[N(T_{\varepsilon})+1].
\end{equation*}

\noindent Hence, by Lemma \ref{tau} and Lemma \ref{N(T)}
\begin{displaymath}
    \mathbb{E}[T_{\varepsilon}] \leq \frac{p}{2p-1}\Big(\frac{2 \zeta F_{max} }{\Theta \varepsilon^2}+2\Big). 
\end{displaymath}

\end{proof}
\subsection{Optimality Condition}
Theorem \ref{optimality} demonstrates that if $||d_k||<\varepsilon$, the $\varepsilon$-optimality condition for $f(x)$ will be satisfied. The error bound on this optimality is linked to both the sample size and the subgradient of $f_k(x)$. Readers seeking the in-depth proof of Theorem \ref{optimality} can refer to Appendix \ref{optimality condition proof}.

\begin{theorem} \label{optimality}
    Let $d^* = - \argmin_{g \in \partial f(\hat{\alpha}_k)} ||g|| $. If $k$ is the smallest index for which $||d_k|| < \varepsilon$ and $|S_k| \geq \frac{-2(\varepsilon')^2}{L^2\log \delta}$, then 
    \begin{equation*}
        P(||d^*|| < 4 \varepsilon + \varepsilon') \geq 1-\delta.
    \end{equation*}
\end{theorem}

\section{Preliminary Computational Results} \label{IT&ER}

Our computational experiments are based on data sets available at the UCI Machine Learning Repository~\cite{UCI2010}. Our study considers three different methods: Wolfe's algorithm \cite{W1975}, kernel Pegasos algorithm\footnote{The computations below are based on our implementation of the Pegasos Algorithm.} and the SCS algorithm. Wolfe's algorithm uses a deterministic objective function (i.e., the SAA) and is treated as a benchmark, while Pegasos and SCS use (\ref{Kernel-SVM}) as the objective function.
However, in order to compare the progress of the objective function  values for different algorithms, we will track the values of (\ref{Kernel-SVM}) for all three algorithms. The values of $\{f(\hat{\alpha}_k)\}_{k=1}^{k=50}$ and $\{f(\hat{\alpha}_k)\}_{k=-50}^{k=-1}$ are shown in Figures $\ref{obj}$ and $\ref{obj1}$, where the notation $k = 1, \ldots, 50$, and $k=-50, \ldots, -1$ represent the solution value in the first fifty, and last fifty (respectively) iterations of each algorithms. The algorithms were implemented on a MacBook Pro 2019 with a 2.6GHz 6-core Intel Core i7 processor and a 16GB of 2400MHz DDR4 onboard memory.

\hspace{2cm}

\begin{figure}[ht]
\footnotesize
     \centering
     \caption{$\{f(\hat{\alpha}_k)\}_{k=1}^{k=50}$ for different combinations (data,algorithm).}
     \begin{subfigure}[b]{0.28\textwidth}
         \centering         \includegraphics[width=\textwidth]{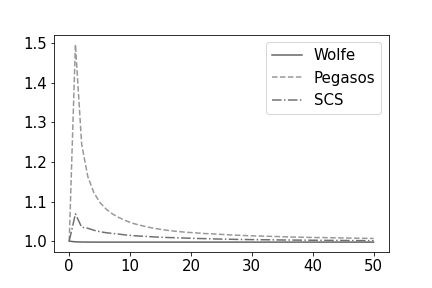}
         \caption{Breast Cancer}
     \end{subfigure}
     \hfill
     \begin{subfigure}[b]{0.28\textwidth}
         \centering
         \includegraphics[width=\textwidth]{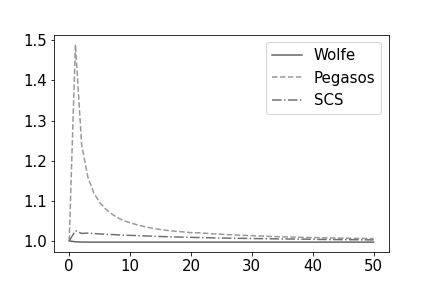}
         \caption{Heart Failure}
     \end{subfigure}
     \hfill
     \begin{subfigure}[b]{0.28\textwidth}
         \centering
         \includegraphics[width=\textwidth]{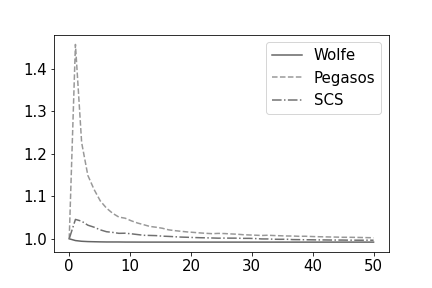}
         \caption{Wine Quality}
     \end{subfigure}\\
     \begin{subfigure}[b]{0.28\textwidth}
         \centering
         \includegraphics[width=\textwidth]{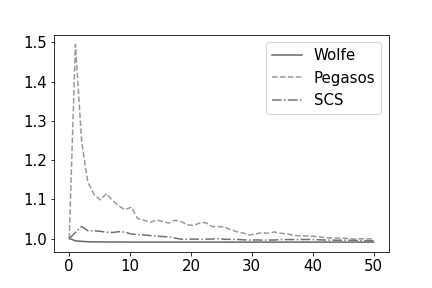}
         \caption{Avila Bible}
     \end{subfigure}
     \hfill
     \begin{subfigure}[b]{0.28\textwidth}
         \centering
         \includegraphics[width=\textwidth]{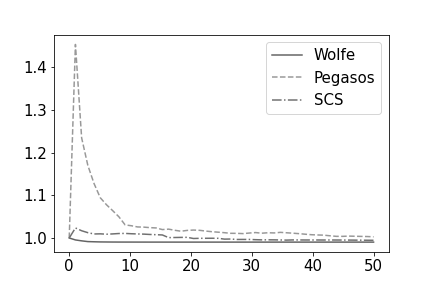}
         \caption{Magic Telescope}
     \end{subfigure}
     \hfill
     \begin{subfigure}[b]{0.28\textwidth}
         \centering
         \includegraphics[width=\textwidth]{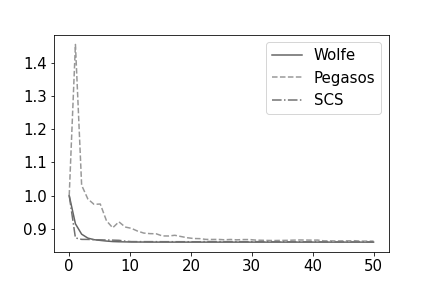}
         \caption{Room Occupancy}
     \end{subfigure}\\
        \label{obj}
\end{figure}

\begin{figure}[ht]
\footnotesize
     \centering
     \caption{$\{f(\hat{\alpha}_k)\}_{k=-50}^{k=-1}$ for different combinations (data,algorithm).}
     \begin{subfigure}[b]{0.28\textwidth}
         \centering
         \includegraphics[width=\textwidth]{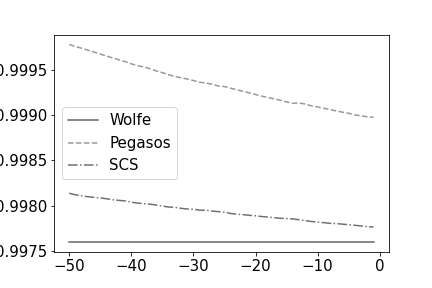}
         \caption{Breast Cancer}
     \end{subfigure}
     \hfill
     \begin{subfigure}[b]{0.28\textwidth}
         \centering
         \includegraphics[width=\textwidth]{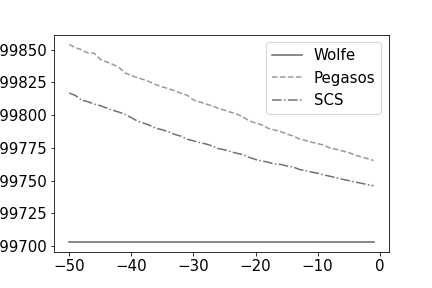}
         \caption{Heart Failure}
     \end{subfigure}
     \hfill
     \begin{subfigure}[b]{0.28\textwidth}
         \centering
         \includegraphics[width=\textwidth]{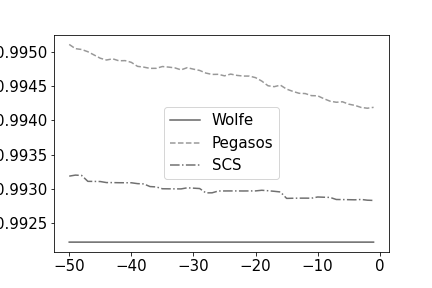}
         \caption{Wine Quality}
     \end{subfigure}\\
     \begin{subfigure}[b]{0.28\textwidth}
         \centering
         \includegraphics[width=\textwidth]{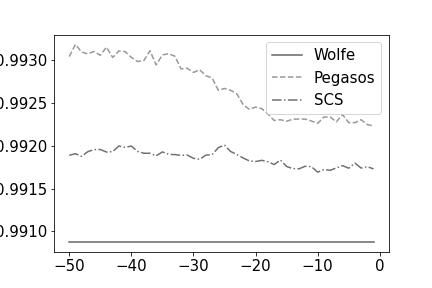}
         \caption{Avila Bible}
     \end{subfigure}
     \hfill
     \begin{subfigure}[b]{0.28\textwidth}
         \centering
         \includegraphics[width=\textwidth]{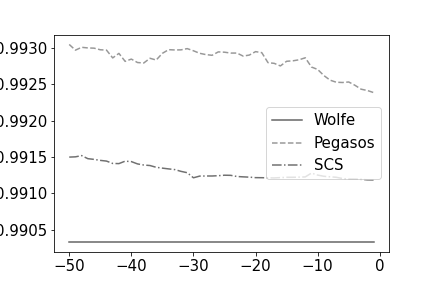}
         \caption{Magic Telescope}
     \end{subfigure}
     \hfill
     \begin{subfigure}[b]{0.28\textwidth}
         \centering
         \includegraphics[width=\textwidth]{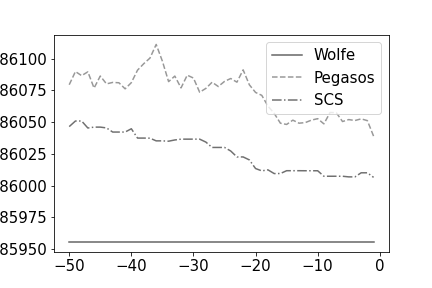}
         \caption{Room Occupancy}
     \end{subfigure}\\
        \label{obj1}
\end{figure}

\noindent \textbf{Remark 3}: (i) The objective function values validated for  the SCS algorithm are lower than those for the Pegasos algorithm. (ii) The SCS algorithm also converges faster than the Pegasos Algorithm. (iii) The beauty of the Pegasos algorithm is that it provides a decision rule without providing a solution to \eqref{Kernel-SVM}. This decision rule provides a classifier without any reference to optimality. On the other hand, the SCS algorithm is truly an optimization algorithm  with a stopping rule in which the steps (and stopping) are guided by $d_k$ ($||d_k||<\varepsilon$).

The performance of all three algorithms for different data sets are given in Table \ref{classification1}. To make the experiment more persuasive, we divide the data into training and testing parts. We obtain a classifier using the training data and apply such classifiers in the testing data set to estimate the accuracy. Here, accuracy reflects the fraction of total number of data points that have been correctly classified.  
Since Wolfe's algorithm requires the entire data set, instances with very large data sets are beyond its reach and are reported by N/A in Table \ref{classification1}. The accuracy and training time reported in the Table \ref{classification1} are based on the average of 20 independent runs for each pair (data, algorithm) using different random seeds.  

\renewcommand{\arraystretch}{1.3}


\begin{table}[h]
\caption{Classification performance of algorithms for different data sets}
\scriptsize
\begin{tabular}{|C{1cm}|cC{1cm}|c|C{0.6cm}c|c|cc|c|cc|}
\hline
                & Algorithm &  & &\multicolumn{2}{c|}{Wolfe}              &  & \multicolumn{2}{c|}{Pegasos}             &  & \multicolumn{2}{c|}{SCS}                \\ \hline
data sets       & \multicolumn{1}{c|}{samples} & features    &  & \multicolumn{1}{c|}{accu.} & time(s) &  & \multicolumn{1}{c|}{accu.} & time(s)  &  & \multicolumn{1}{c|}{accu.} & time(s) \\ \hline
heart failure    & \multicolumn{1}{c|}{273} & 13        &  & \multicolumn{1}{c|}{0.833}    & 1.48   &  & \multicolumn{1}{c|}{0.833}    & 0.75     &  & \multicolumn{1}{c|}{0.833}    & 6.12   \\ \hline
breast cancer   & \multicolumn{1}{c|}{500} & 30        &  & \multicolumn{1}{c|}{0.97}     & 2.28   &  & \multicolumn{1}{c|}{0.95}     & 1.14    &  & \multicolumn{1}{c|}{0.97}     & 4.92   \\ \hline
wine quality    & \multicolumn{1}{c|}{680} & 11        &  & \multicolumn{1}{c|}{0.87}     & 4.19   &  & \multicolumn{1}{c|}{0.855}    & 3.93    &  & \multicolumn{1}{c|}{0.86}     & 12.49  \\ \hline
avila Bible     & \multicolumn{1}{c|}{2,000} & 10      &  & \multicolumn{1}{c|}{0.726}    & 32.97  &  & \multicolumn{1}{c|}{0.714}    & 5.60    &  & \multicolumn{1}{c|}{0.718}    & 38.26  \\ \hline
magic telescope & \multicolumn{1}{c|}{5,000} & 10      &  & \multicolumn{1}{c|}{N/A}      & N/A     &  & \multicolumn{1}{c|}{0.735}    & 25.99   &  & \multicolumn{1}{c|}{0.74}     & 33.25  \\ \hline
room occupancy  & \multicolumn{1}{c|}{7,500} & 18      &  & \multicolumn{1}{c|}{N/A}      & N/A     &  & \multicolumn{1}{c|}{0.974}    & 56.54   &  & \multicolumn{1}{c|}{0.976}    & 39.88  \\ \hline
swarm behavious & \multicolumn{1}{c|}{20,000} & 2400     &  & \multicolumn{1}{c|}{N/A}      & N/A     &  & \multicolumn{1}{c|}{0.805}    & 67.61   &  & \multicolumn{1}{c|}{0.9}      & 13.38   \\ \hline
mini-BooNE      & \multicolumn{1}{c|}{120,000} & 50    &  & \multicolumn{1}{c|}{N/A}      & N/A     &  & \multicolumn{1}{c|}{0.79}     & 145.86  &  & \multicolumn{1}{c|}{0.81}     & 39.57  \\ \hline
skin-nonskin    & \multicolumn{1}{c|}{200,000} & 3    &  & \multicolumn{1}{c|}{N/A}      & N/A     &  & \multicolumn{1}{c|}{0.93}     & 176.45  &  & \multicolumn{1}{c|}{0.97}     & 19.68  \\ \hline
Hepmass         & \multicolumn{1}{c|}{3,500,000} & 28  &  & \multicolumn{1}{c|}{N/A}      & N/A     &  & \multicolumn{1}{c|}{0.83}     & 3762.89 &  & \multicolumn{1}{c|}{0.87}     & 42.69  \\ \hline
\end{tabular}
\label{classification1}%
\end{table}

\noindent \textbf{Remark 4}: Conventional wisdom in ML suggests that SFO methods provide the best balance between scalability and accuracy of stochastic optimization problems. However from Table \ref{classification1}, we can see that this general impression may not only be questionable but also encourage an over-emphasis on SFO methods. In fact, the SCS algorithm performs better than the Pegasos algorithm both in accuracy and in computational speed, especially for the larger data sets in our experiment. Specifically, for the four largest data sets in our experiment, the computational times and accuracy reported for the SCS algorithm are superior to those obtained using Pegasos. Thus, the SCS algorithm gives an alternative approach which is quite competitive for Kernel SVM problems.    

\section{Conclusion}\label{Con}
Our approach leads to a class of online algorithms that go beyond first-order approximations and incorporates stochastic  analysis of stopping times in a direction-finding method. Specifically, it includes several features of Wolfe's method (e.g., stability and good convergence rate) as well as online aspects of Pegasos. This combination promotes improvements on both deterministic and stochastic methods that have been previously used for SVM. In contrast to the Pegasos algorithm, we find that the optimization performance of the SCS algorithm is more reliable and provides consistently lower objective values (see Figures \ref{obj} and \ref{obj1}) in out-of-sample validation. It appears that such reliability may be difficult to achieve using first-order methods without additional fine-tuning. Although our experimental results are based on some fixed data sets, the SCS approach can be effective for classification of streaming data as well. Finally, we should emphasize that although the SCS algorithm in this paper is specifically designed for kernel SVM, it can be generalized to other non-smooth stochastic convex optimization problems, as well as constraimed problems using primal-dual methods similar to \cite{Sen1986ACO, R1991}. We also conjecture that many of the asymptotic properties (e.g., asymptotic normality) of SGD methods may also be studied for the SCS algorithm \cite {davis2023}.  \\

\noindent \textbf{Acknowledgment}. This paper is supported, in part, by a grant from AFOSR FA9550-20-1-0006.  The final phase of this paper was funded via ONR grant N00014-20-1-2077.

\bibliographystyle{siamplain}
\bibliography{main}

\begin{thebibliography}{10}

\bibitem{B2014}
{\sc A.~S. Bandeira, K.~Scheinberg, and L.~N. Vicente}, {\em Convergence of trust-region methods based on probabilistic models}, SIAM Journal on Optimization, 24 (2014), pp.~1238--1264.

\bibitem{BL2011}
{\sc J.~R. Birge and F.~Louveaux}, {\em Introduction to stochastic programming}, Springer Science \& Business Media, 2011.

\bibitem{BS2019}
{\sc J.~Blanchet, C.~Cartis, M.~Menickelly, and K.~Scheinberg}, {\em Convergence rate analysis of a stochastic trust-region method via supermartingales}, INFORMS journal on optimization, 1 (2019), pp.~92--119.

\bibitem{BB2011}
{\sc L.~Bottou and O.~Bousquet}, {\em 13 the tradeoffs of large-scale learning}, Optimization for Machine Learning,  (2011), p.~351.

\bibitem{CS2018}
{\sc C.~Cartis and K.~Scheinberg}, {\em Global convergence rate analysis of unconstrained optimization methods based on probabilistic models}, Mathematical Programming, 169 (2018), pp.~337--375.

\bibitem{C2018}
{\sc R.~Chen, M.~Menickelly, and K.~Scheinberg}, {\em Stochastic optimization using a trust-region method and random models}, Mathematical Programming, 169 (2018), pp.~447--487.

\bibitem{CScheinberg2017}
{\sc F.~E. Curtis and K.~Scheinberg}, {\em Optimization methods for supervised machine learning: From linear models to deep learning}, in Leading Developments from INFORMS Communities, INFORMS, 2017, pp.~89--114.

\bibitem{davis2023}
{\sc D.~Davis, D.~Drusvyatskiy, and L.~Jiang}, {\em Asymptotic normality and optimality in nonsmooth stochastic approximation}, arXiv preprint arXiv:2301.06632,  (2023).

\bibitem{F2019}
{\sc M.~Fazlyab, A.~Robey, H.~Hassani, M.~Morari, and G.~Pappas}, {\em Efficient and accurate estimation of lipschitz constants for deep neural networks}, Advances in Neural Information Processing Systems, 32 (2019).

\bibitem{UCI2010}
{\sc A.~Frank}, {\em Uci machine learning repository}, http://archive. ics. uci. edu/ml,  (2010).

\bibitem{GGD2020}
{\sc G.~Grimmett and D.~Stirzaker}, {\em Probability and random processes}, Oxford university press, 2020.

\bibitem{HS1991}
{\sc J.~L. Higle and S.~Sen}, {\em Stochastic decomposition: An algorithm for two-stage linear programs with recourse}, Mathematics of operations research, 16 (1991), pp.~650--669.

\bibitem{HS1994}
{\sc J.~L. Higle and S.~Sen}, {\em Finite master programs in regularized stochastic decomposition}, Mathematical Programming, 67 (1994), pp.~143--168.

\bibitem{J2018}
{\sc X.-B. Jin, X.-Y. Zhang, K.~Huang, and G.-G. Geng}, {\em Stochastic conjugate gradient algorithm with variance reduction}, IEEE Transactions on Neural Networks and Learning Systems, 30 (2018), pp.~1360--1369.

\bibitem{J1999}
{\sc T.~Joachims}, {\em Making large-scale support vector machine learning practical, advances in kernel methods}, Support vector learning,  (1999).

\bibitem{KG1971}
{\sc G.~Kimeldorf and G.~Wahba}, {\em Some results on tchebycheffian spline functions}, Journal of mathematical analysis and applications, 33 (1971), pp.~82--95.

\bibitem{L2023}
{\sc G.~Lan}, {\em Stochastic gradient descent}, in Encyclopedia of Optimization, Springer, 2023, pp.~1--3.

\bibitem{LS2022}
{\sc J.~Liu, G.~Li, and S.~Sen}, {\em Coupled learning enabled stochastic programming with endogenous uncertainty}, Mathematics of Operations Research, 47 (2022), pp.~1681--1705.

\bibitem{MJ2010}
{\sc J.~Martens et~al.}, {\em Deep learning via hessian-free optimization.}, in ICML, vol.~27, 2010, pp.~735--742.

\bibitem{NB2001}
{\sc A.~Nedi{\'c} and D.~Bertsekas}, {\em Convergence rate of incremental subgradient algorithms}, in Stochastic optimization: algorithms and applications, Springer, 2001, pp.~223--264.

\bibitem{N2006}
{\sc J.~Nocedal and S.~Wright}, {\em Numerical Optimization}, Springer Science \& Business Media, 2006.

\bibitem{PS2020}
{\sc C.~Paquette and K.~Scheinberg}, {\em A stochastic line search method with expected complexity analysis}, SIAM Journal on Optimization, 30 (2020), pp.~349--376.

\bibitem{P2010}
{\sc R.~Pasupathy}, {\em On choosing parameters in retrospective-approximation algorithms for stochastic root finding and simulation optimization}, Operations Research, 58 (2010), pp.~889--901.

\bibitem{P1998}
{\sc J.~Platt}, {\em Sequential minimal optimization: A fast algorithm for training support vector machines},  (1998).

\bibitem{polyak1992}
{\sc B.~T. Polyak and A.~B. Juditsky}, {\em Acceleration of stochastic approximation by averaging}, SIAM journal on control and optimization, 30 (1992), pp.~838--855.

\bibitem{RM1951}
{\sc H.~Robbins and S.~Monro}, {\em A stochastic approximation method}, The annals of mathematical statistics,  (1951), pp.~400--407.

\bibitem{R1991}
{\sc R.~T. Rockafellar and R.~J.-B. Wets}, {\em Scenarios and policy aggregation in optimization under uncertainty}, Mathematics of operations research, 16 (1991), pp.~119--147.

\bibitem{scholkopf2002}
{\sc B.~Sch{\"o}lkopf and A.~J. Smola}, {\em Learning with kernels: support vector machines, regularization, optimization, and beyond}, MIT press, 2002.

\bibitem{Sen1986ACO}
{\sc S.~Sen and H.~D. Sherali}, {\em A class of convergent primal-dual subgradient algorithms for decomposable convex programs}, Mathematical Programming, 35 (1986), pp.~279--297.

\bibitem{SS2011}
{\sc S.~Shalev-Shwartz, Y.~Singer, N.~Srebro, and A.~Cotter}, {\em Pegasos: Primal estimated sub-gradient solver for svm}, Mathematical programming, 127 (2011), pp.~3--30.

\bibitem{SZ2013}
{\sc S.~Shalev-Shwartz and T.~Zhang}, {\em Stochastic dual coordinate ascent methods for regularized loss minimization.}, Journal of Machine Learning Research, 14 (2013).

\bibitem{Shapiro2003}
{\sc A.~Shapiro}, {\em Monte carlo sampling methods}, Handbooks in operations research and management science, 10 (2003), pp.~353--425.

\bibitem{W1974}
{\sc P.~Wolfe}, {\em Note on a method of conjugate subgradients for minimizing nondifferentiable functions}, Mathematical Programming, 7 (1974), pp.~380--383.

\bibitem{W1975}
{\sc P.~Wolfe}, {\em A method of conjugate subgradients for minimizing nondifferentiable functions}, in Nondifferentiable optimization, Springer, 1975, pp.~145--173.

\bibitem{W1996}
{\sc G.~Wood and B.~Zhang}, {\em Estimation of the lipschitz constant of a function}, Journal of Global Optimization, 8 (1996), pp.~91--103.

\bibitem{Y2022}
{\sc Z.~Yang}, {\em Adaptive stochastic conjugate gradient for machine learning}, Expert Systems with Applications,  (2022), p.~117719.

\bibitem{YNShanbhag2020}
{\sc F.~Yousefian, A.~Nedic, and U.~V. Shanbhag}, {\em On stochastic and deterministic quasi-newton methods for nonstrongly convex optimization: Asymptotic convergence and rate analysis}, SIAM Journal on Optimization, 30 (2020), pp.~1144--1172.

\bibitem{Y2016}
{\sc F.~X.~X. Yu, A.~T. Suresh, K.~M. Choromanski, D.~N. Holtmann-Rice, and S.~Kumar}, {\em Orthogonal random features}, Advances in neural information processing systems, 29 (2016), pp.~1975--1983.

\end{thebibliography}

\appendix

\section{Proof of optimality condition} \label{optimality condition proof}
In this section, we want to prove that $||d_k|| < \varepsilon$ implies that the $\varepsilon$-optimality condition of $f(\hat{\alpha}_k)$, as stated in Theorem \ref{optimality}. The error bound on optimality is related to the sample size as well as the the subgradient of $f_k(\hat{\alpha}_k)$. First of all, we will define a few well-known notations.

\begin{itemize}
    \item $d^* = - \argmin_{g \in \partial f(\hat{\alpha}_k)} ||g|| $: negative of the smallest norm subgradient of $f(\hat{\alpha}_k)$.
    \item $d_k^* = - \argmin_{g \in \partial f_k(\hat{\alpha}_k)} ||g|| $: negative of the smallest norm subgradient of $f_k(\hat{\alpha}_k)$.
    \item $d_k = -Nr(\{g_k,-d_{k-1} \})$, where $g_k \in \partial f_k(\hat{\alpha}_k)$.
    \item Subdifferential of $f$ at $z'$: $\partial f(z') = \{q: f(z) \geq f(z') + q^T (z-z')\}$.
    \item $\varepsilon$-subdifferential of $f$ at $z'$: $\partial f_{\varepsilon}(z') = \{q: f(z) \geq f(z') + q^T (z-z')-\varepsilon\}$.
    \item Directional derivative of $f$ at $z$ in direction $d$: $f'(z;d)=\max\{d^Tq: q \in \partial f(z)\}$.
    \item $\varepsilon$-directional derivative of $f$ at $z$ in direction $d$: $f_{\varepsilon}'(z;d)=\max\{d^Tq: q \in \partial_{\varepsilon} f(z)\}$.
    \item $\mathcal{X}(z) := -\min_{||d|| = 1} f'(z;d)$ and $\mathcal{X}_{k}(z) := -\min_{||d|| = 1} f_k'(z;d)$.
    \item $\mathcal{X}_{\varepsilon}(z) := -\min_{||d|| = 1} f_{\varepsilon}'(z;d)$ and $\bar{d}_k(z) = - \argmin_{||d|| = 1} f_{\varepsilon}'(z;d) $.
\end{itemize}

\begin{lemma} \label{d_k^*}
    Let the tuple $(m_1, m_2)$ satisfy the requirement in Algorithm \ref{LSA} ($\frac{1}{4} \leq m_1 < \frac{1}{2}$ and $\frac{1}{4} \leq m_2 < m_1$). If $k$ is the smallest index for which $||d_k|| < \varepsilon$, then we have $||d_k^*|| < 4 \varepsilon$.
\end{lemma}

\begin{proof}
    Suppose the claim is false, then by the definition of $d_k^*$, we have $||g_k|| \geq ||d_k^*|| \geq 4\varepsilon$. Thus,
    \begin{equation} \label{norm of d_k}
        \begin{aligned}
            ||d_k||^2 & = ||\lambda_k^* g_k-(1-\lambda_k^*)d_{k-1}||^2\\
            & = (\lambda_k^*)^2||g_k||^2+(1-\lambda_k^*)^2||d_{k-1}||^2-2\lambda_k^*(1-\lambda_k^*)\langle g_k,d_{k-1} \rangle \\
            & \geq (\lambda_k^*)^2||g_k||^2+(1-\lambda_k^*)^2||d_{k-1}||^2 \\
        \end{aligned}
    \end{equation}

\noindent where the inequality holds because Algorithm \ref{LSA} ensures the condition R in \eqref{L&R}. 

\noindent We will now divide the analysis into 2 cases which are examined below: a) $||g_k|| > ||d_{k-1}|| \geq \varepsilon$ and $||g_k|| \geq 4 \varepsilon$  and b) $||d_{k-1}|| \geq ||g_k|| \geq 4 \varepsilon$. 

(a) In this case from \eqref{norm of d_k} we have 
\begin{equation} \label{norm of d_k case 1}
    ||d_k||^2 \geq \big(17 (\lambda_k^*)^2 -2 \lambda_k^* + 1 \big) \varepsilon^2
\end{equation}

\noindent Note that
\begin{equation*}
    \lambda_k^*=\frac{\langle g_k,d_{k-1} \rangle + ||g_k||^2}{||d_{k-1}||^2+||g_k||^2+2\langle g_k,d_{k-1} \rangle} \geq \frac{-m_2||d_{k-1}||^2+||g_k||^2}{||d_{k-1}||^2+||g_k||^2},
\end{equation*}

\noindent where the inequality holds because the R condition in \eqref{L&R} ensures $0 > \langle g_k,d_{k-1} \rangle \geq - m_2 ||d_{k-1}||^2$. Thus, we claim that $\lambda_k^* \geq \frac{2}{17}$ because it is suffice to show

\begin{equation*}
    - 17 \cdot m_2||d_{k-1}||^2+17 \cdot ||g_k||^2 \geq 2 ||d_{k-1}||^2 + 2 ||g_k||^2.
\end{equation*}

\noindent This can be verified by observing $m_2 < m_1 < \frac{1}{2}$ and $||g_k|| > ||d_{k-1}||$. Thus, based on \eqref{norm of d_k case 1} and $\lambda_k^* \geq \frac{2}{17}$, we have
$||d_k|| \geq \varepsilon $. This contradicts the assumptions of the lemma. 

\noindent (b) If $||d_{k-1}|| \geq ||g_k|| \geq 4 \varepsilon$, then we have 
\begin{equation*}
    ||d_k||^2 \geq \big(32 (\lambda_k^*)^2 -32 \lambda_k^* + 16 \big) \varepsilon^2 \geq 8 \varepsilon^2. 
\end{equation*}

\noindent Thus, $||d_k|| \geq 2 \sqrt{2} \varepsilon $, which also contradicts the assumptions of the lemma.
\end{proof}

\begin{lemma} \label{d and X}
    $||d^*|| = \mathcal{X}(\hat{\alpha}_k)$ and $||d_k^*|| = \mathcal{X}_k(\hat{\alpha}_k)$.
\end{lemma}

\begin{proof}
    According to the definition of $\mathcal{X}(\hat{\alpha}_k)$,
    \begin{equation} \label{min-max}
        \mathcal{X}(\hat{\alpha}_k)=-\min_{||d|| = 1} f'(\hat{\alpha}_k;d) = -\min_{||d|| = 1} \max_{q \in \partial f(\hat{\alpha}_k)}d^Tq
    \end{equation}

\noindent Note that the min-max  problem \eqref{min-max} has a stationary point

\begin{equation*}
    \bar{q}=\argmin_{q \in \partial f(\hat{\alpha}_k)} ||q||, \quad  \bar{d} = - \frac{\bar{q}}{||\bar{q}||}.
\end{equation*}

\noindent Thus, $\mathcal{X}(\hat{\alpha}_k) = ||\bar{q}|| = ||d^*||$. Similarly, we can also prove that $ \mathcal{X}_k(\hat{\alpha}_k) = ||d_k^*||$.
\end{proof}

\begin{theorem} \label{X and X_k}
    If $|S_k| \geq \frac{-2(\varepsilon')^2}{L^2\log \delta}$, then $P(\mathcal{X}_{\varepsilon}(\hat{\alpha}_k)-\mathcal{X}_k(\hat{\alpha}_k) \leq \varepsilon') \geq 1-\delta$.
\end{theorem}

\begin{proof}
By definition, $\bar{d}_k(\hat{\alpha}_k) = - \argmin_{||d|| = 1} f_{\varepsilon}'(\hat{\alpha}_k;d) $. Hence,
    \begin{equation*}
        \begin{aligned}
        \mathcal{X}_{\varepsilon}(\hat{\alpha}_k)-\mathcal{X}_{k}(\hat{\alpha}_k)
        & \leq \frac{1}{|S_k|}\sum_{i=1}^{|S_k|} h'(\hat{\alpha}_k;\bar{d}_k(\hat{\alpha}_k),\omega_i)-E[h_{\varepsilon}'(\hat{\alpha}_k;\bar{d}_k(\hat{\alpha}_k),\omega)]\\
        & = \frac{1}{|S_k|}\sum_{i=1}^{|S_k|} [h'(\hat{\alpha}_k;\bar{d}_k(\hat{\alpha}_k),\omega_i)-h_{\varepsilon}'(\hat{\alpha}_k;\bar{d}_k(\hat{\alpha}_k),\omega_i)]\\
        & + \frac{1}{|S_k|}\sum_{i=1}^{|S_k|} h_{\varepsilon}'(\hat{\alpha}_k;\bar{d}_k(\hat{\alpha}_k),\omega_i) - E[h_{\varepsilon}'(\hat{\alpha}_k;\bar{d}_k(\hat{\alpha}_k),\omega)]
        \end{aligned}
    \end{equation*}
\noindent Note that for each $\omega_i$, we have $h'(\hat{\alpha}_k;\bar{d}_k(\hat{\alpha}_k),\omega_i) \leq h_{\varepsilon}'(\hat{\alpha}_k;\bar{d}_k(\hat{\alpha}_k),\omega_i)$, since $\partial h(\hat{\alpha}_k,\omega_i) \subseteq \partial h_\varepsilon(\hat{\alpha}_k,\omega_i) $. Also, by Hoeffding's inequality, if $|S_k| \geq \frac{-2(\varepsilon')^2}{L^2\log \delta}$, then 

\begin{equation*}
    P\big(\frac{1}{|S_k|}\sum_{i=1}^{|S_k|} h_{\varepsilon}'(\hat{\alpha}_k;\bar{d}_k(\hat{\alpha}_k),\omega_i) - E[h_{\varepsilon}'(\hat{\alpha}_k;\bar{d}_k(\hat{\alpha}_k),\omega)] \leq \varepsilon'\big) \geq 1-\delta,
\end{equation*}

\noindent which concludes the proof.
\end{proof}

\begin{proof} (Theorem \ref{optimality})
    First, by Theorem \ref{X and X_k} and Lemma \ref{d and X}, if $|S_k| \geq \frac{-2(\varepsilon')^2}{L^2\log \delta}$, then with probability at least $1-\delta$,
    \begin{equation*}
        ||d^*||-||d_k^*||=\lim_{\varepsilon \to 0} \mathcal{X}_{\varepsilon} (\hat{\alpha}_k) -\mathcal{X}_k(\hat{\alpha}_k) \leq \varepsilon'.
    \end{equation*}
    \noindent Combining with Lemma \ref{d_k^*}, we have
    \begin{equation*}
        P \big(||d^*|| < 4 \varepsilon + \varepsilon' \big) \geq 1 - \delta.
    \end{equation*}
\end{proof}

\section{Notation Tables}\label{notations}
\renewcommand{\arraystretch}{1.1}
\begin{table}[h]
    \footnotesize
    \centering
    \begin{tabular}{|c|l|}
    \hline
        Symbols & Definitions \\
    \hline
         $m$ & Number of samples\\
    \hline
        $S=\left\{ (z_i,w_i) \right\}_{i=1}^{m}$ & The dataset with pairs correspond to features and their associated labels\\
    \hline
        $(Z,W)$ & The random counterpart of $(z,w)$ \\
    \hline
        $\phi$ & The kernel mapping which maps the feature $z$ to a higher dimension\\
    \hline
        $\beta$ & The decision variable (classifier) for the SVM problem \\
    \hline
        $\alpha$ & The decision variable for kernel SVM problem\\
    \hline
        $Q \in \mathbb{R}^{m \times m}$ & The kernel matrix with $ Q_{ij} = \langle \phi(z_i), \phi(z_j) \rangle$ \\
    \hline 
        $Q_i$ & The $i^{th}$ row of $Q$ \\
    \hline
        $Q^k$ & The kernel matrix generated at $k^{th}$ iteration \\
    \hline
        $Q_i^k$ & The $i^{th}$ row of $Q^k$\\
    \hline 
        $S_k$ & The dataset used at $k^{th}$ iteration\\
    \hline 
        $|S_k|$ & The number of data points in $S_k$ \\
    \hline 
        $0 < L < \infty$ & The Lipschitz constant for function $f$ and $f_k$\\
    \hline 
        $M$ & A large constant to bound $|\langle \phi(z), \beta \rangle|$ for any $\phi(z)$ and $\beta$ \\
    \hline 
    \end{tabular}
    \caption{Definitions of Symbols in Section 1}
\end{table}

\renewcommand{\arraystretch}{1.3}
\begin{table}[h]
    \footnotesize
    \centering
    \begin{tabular}{|c|l|}
    \hline
        Symbols & Definitions \\
    \hline
        $d_k$ & The search direction at $k^{th}$ iteration\\
    \hline
        $g_k$ & The subgradient of $f_k$ at $k^{th}$ iteration\\
    \hline
        $G_k = \{d_{k-1},g_k\}$ & A convex set which contains all of the convex combinations of $d_{k-1}$ and $g_k$ \\
    \hline
        $Nr(G_k)$ & A function which returns a vector with the smallest norm in the convex hall of $G_k$ \\
    \hline
        $\varepsilon$ & An algorithmic parameter related to the termination criteria ($||d_k||<\varepsilon$) \\
    \hline
        $\delta$ & The region where the line-search is executed\\
    \hline 
        $\delta_0$, $\delta_{min}$ and $\delta_{max}$ & The initial, minimum and maximum choice of the region\\
    \hline
        $\eta_1 \in (0,1)$ & An algorithmic parameter to identify whether $f_k$ and $\hat{f}_k$ have similar decrease   \\
    \hline 
        $\eta_2 > 0$ & An algorithmic parameter to identify whether $||d_k||$ is large enough\\
    \hline
        $\gamma>1$ & An algorithmic parameter to adjust the line-search region $\delta$ \\
    \hline
        $\alpha_k$ & The candidate solution of $f_{k-1}$\\
    \hline
        $\hat{\alpha}_k$ & The optimal(incumbent) solution of $f_{k-1}$\\
    \hline
        $\beta_k$ & The candidate solution of $f_k$\\
    \hline
        $\hat{\beta}_k$ & The optimal(incumbent) solution of $f_k$\\
    \hline 
        $T_k$ & The dataset used in the $k^{th}$ iteration independent from $S_k$\\
    \hline 
        $|T_k|$ & The number of data points in $T_k$ \\
    \hline 
        $\frac{1}{4}<m_1<0.5$ & An algorithmic parameter to identify whether $f_k$ has sufficient decrease\\
    \hline 
        $m_2(<m_1)$ & An algorithmic parameter to identify whether the directional derivative of $f_k$ has sufficient increase \\
    \hline 
        $1< n \in \mathbb{Z}^+$ & An algorithmic parameter to avoid unlimited iterations of line-search\\
    \hline
        $\zeta$ & An algorithmic parameter to bound the line-search region\\
    \hline
    \end{tabular}
    \caption{Definitions of Symbols in Section 2}
\end{table}

\renewcommand{\arraystretch}{1.3}
\begin{table}[h]
    \centering
    \footnotesize
    \begin{tabular}{|c|l|}
    \hline
        Symbols & Definitions \\
    \hline
        $\kappa$ & The $\kappa$-approximation parameter (see Definition 3.1)\\ 
    \hline
        $I_k$ & An event when $f_k$ satisfies $\kappa$-approximation property\\
    \hline
        $\mu_1$ & The probability of event $I_k$\\
    \hline
        $\varepsilon_F$ & The $\varepsilon_F$-accurate parameter (see Definition 3.4)\\
    \hline
        $J_k $ & A event when $\hat{f}_k$ is $\varepsilon_F$-accurate \\
    \hline
        $\mu_2$ & The probability of event $J_k$ \\
    \hline
        $T_{\varepsilon} = \inf \{ k > 0: ||d_k|| < \varepsilon\}$ & The first time that $||d_k||<\varepsilon$ \\
    \hline
        $A_l$ & The renewal process which renews once $\delta_l > \frac{\varepsilon}{\zeta}$\\
    \hline
        $\tau_l = A_l - A_{l-1}$ & The inter-arrival time\\
    \hline 
        $N(k) = \max \{l: A_l \leq k\}$ & The counting process representing the number of renewals occurred up to $k^{th}$ iteration\\
    \hline
    \end{tabular}
    \caption{Definitions of Symbols in Section 3}
\end{table}
\end{document}